\documentclass{article}

\usepackage{microtype}
\usepackage{graphicx}
\usepackage{subcaption}
\usepackage{booktabs} % for professional tables

\usepackage{hyperref}

\usepackage{enumitem}
\usepackage{multirow}
\usepackage{multicol}
\usepackage{makecell}
\usepackage{adjustbox}
\usepackage{soul}
\sethlcolor{yellow}
\usepackage{tipa}
\usepackage[dvipsnames]{xcolor}
\usepackage{tcolorbox}
\usepackage{colortbl}
\tcbuselibrary{breakable}
\usepackage[fixed]{fontawesome5}
\usepackage{tabularx} 
\usepackage{ragged2e} 
\usepackage{array}
\usepackage{float}
\usepackage{wrapfig}
\usepackage{siunitx}
\let\OriginalAddContentsLine\addcontentsline

\newcolumntype{L}{>{\RaggedRight\arraybackslash}X}
\newcommand{\Cell}[1]{%
  \begin{minipage}[t]{\hsize}%
    \RaggedRight%              % 
    #1%                        % 
  \end{minipage}%
}

\usepackage[accepted]{icml2026}

\usepackage{amsmath}
\usepackage{amssymb}
\usepackage{mathtools}
\usepackage{amsthm}

\usepackage[capitalize,noabbrev]{cleveref}

\theoremstyle{plain}
\newtheorem{theorem}{Theorem}[section]

\newtheorem{lemma}[theorem]{Lemma}
\newtheorem{corollary}[theorem]{Corollary}
\theoremstyle{definition}
\newtheorem{definition}[theorem]{Definition}

\theoremstyle{remark}

\usepackage[textsize=tiny]{todonotes}

\icmltitlerunning{Hierarchical Abstract Tree for Cross-Document Retrieval-Augmented Generation}

\begin{document}

\twocolumn[
  \icmltitle{Hierarchical Abstract Tree for Cross-Document Retrieval-Augmented Generation}



  \icmlsetsymbol{equal}{*}

  \begin{icmlauthorlist}
    % \icmlauthor{Firstname1 Lastname1}{equal,yyy}
    % \icmlauthor{Firstname2 Lastname2}{equal,yyy,comp}
    % \icmlauthor{Firstname3 Lastname3}{comp}
    \icmlauthor{Ziwen Zhao}{hkustgz}
    \icmlauthor{Menglin Yang}{hkustgz}
  \end{icmlauthorlist}

  \icmlaffiliation{hkustgz}{The Hong Kong University of Science and Technology (Guangzhou), Guangzhou, China}

  \icmlcorrespondingauthor{Menglin Yang}{menglinyang@hkust-gz.edu.cn}

  \icmlkeywords{Machine Learning, ICML}

  \vskip 0.3in
]



\printAffiliationsAndNotice{}  % no special notice (required even if empty)

\begin{abstract}
Retrieval-augmented generation (RAG) enhances large language models with external knowledge, and tree-based RAG organizes documents into hierarchical indexes to support queries at multiple granularities.
However, existing Tree-RAG methods designed for single-document retrieval face critical challenges in scaling to cross-document multi-hop questions: \textit{(1) poor distribution adaptability}, where $k$-means clustering introduces noise due to rigid distribution assumptions; \textit{(2) structural isolation}, as tree indexes lack explicit cross-document connections; and \textit{(3) coarse abstraction}, which obscures fine-grained details. To address these limitations, we propose {\bf $\Psi$-RAG}, a tree-RAG framework with two key components. \textit{First}, a hierarchical abstract tree index 
built through an iterative ``merging and collapse'' process that adapts to data distributions without a priori assumption. 
\textit{Second}, a multi-granular retrieval agent that intelligently interacts with the knowledge base with reorganized queries and an agent-powered hybrid retriever.
$\Psi$-RAG supports diverse tasks from token-level question answering to document-level summarization. 
On cross-document multi-hop QA benchmarks, it outperforms RAPTOR by 25.9\% and HippoRAG 2 by 7.4\% in average F1 score\footnote{Code is available at \url{https://github.com/Newiz430/Psi-RAG.}}.
\end{abstract}

\section{Introduction}

\begin{figure}[t]
  \vskip -0.1in
  \begin{center}
    \centerline{\includegraphics[width=\columnwidth]{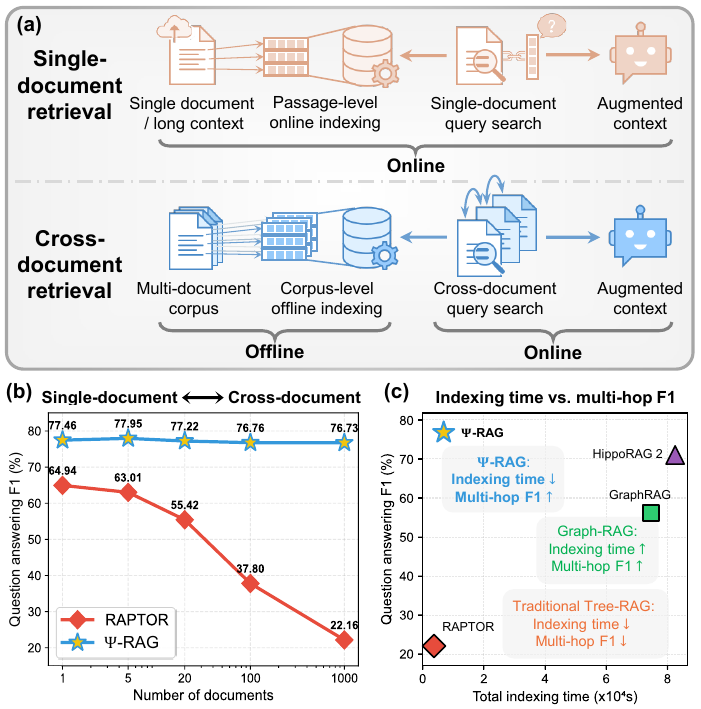}}
    \caption{
(a) Two RAG application scenarios: single-document and cross-document retrieval. 
(b) Multi-hop QA performance {\it w.r.t.} the number of documents in the tree index. 
(c) $\Psi$-RAG has both efficient indexing and accurate multi-hop retrieval compared to structured RAG baselines, using 2Wiki as an example. }
    \label{fig:corpusindex}
  \end{center}
  \vskip -0.3in
\end{figure}

Retrieval-augmented generation (RAG) \cite{rag} enhances large language models (LLMs) with reliable references from external knowledge bases. 
To effectively address multi-granular user queries from token-level factual question answering (QA) to document-level summarization, RAG frameworks must leverage the semantic hierarchy within corpora \cite{raptor, densex}. 
Traditional retrieval systems \cite{bm25, dpr} slice documents into short, independent chunks; this fixed retrieval granularity leads to incomplete and inaccurate responses for higher-level questions \cite{hiercrf, dhr, raptor, archrag}.

{\it Structured RAG} achieves multi-granularity by organizing documents relationally with knowledge graph indexes \cite{graphrag, hipporag, hipporag2} or hierarchically with tree indexes \cite{dhr, longrefiner, raptor}. 
Its utility is demonstrated in two key scenarios, as in \cref{fig:corpusindex}(a): (1) {\it single-document retrieval}, to process user requests in diverse granularity within a single user-uploaded document \cite{longcontextllm1,longcontextllm2,gor}; 
(2) {\it cross-document retrieval}, to process user requests by building a large index from an entire (domain-specific) corpus and retrieving across multiple documents, challenging the multi-hop retrieval and reasoning capabilities of RAG frameworks \cite{ircot, multihoprag}. 

Our work focuses on {\it Tree-RAG} \cite{dhr, longrefiner, raptor, treerag, memtree}, which organizes documents into a tree-structured index. 
It explicitly defines different information granularities, enabling superior retrieval performance, especially for summative tasks \cite{raptor, graphragbench}. See Appendix \ref{relatedwork} for more related works. However, existing Tree-RAG methods mainly target single-hop retrieval on a single long document, while cross-document and multi-hop retrieval remain underexplored. For example, representative $k$-means-type cluster trees like RAPTOR \cite{raptor} achieve high retrieval performance only with small-scale, passage-level indexes. As shown in \cref{fig:corpusindex}(b), {\bf RAPTOR's retrieval accuracy drops significantly when the search space expands to corpus-level with millions of tokens}. We identify three key limitations: 
(1) {\it Poor distribution adaptability}: $k$-means-type clustering implicitly relies on spherical data distributions \cite{raptor}, which introduces noisy documents to the retriever for corpora with a skewed distribution. 
(2) {\it Structural isolation}: 
unlike Graph-RAG which dynamically hops between documents based on pairwise relationships, leaf nodes in a tree index lack explicit connections. 
This prevents the retriever from capturing implicit causal dependencies in multi-hop questions. 
(3) {\it Coarse abstraction}: the coarse-grained abstracts act like a mosaic, obscuring token-level details at the very beginning of retrieval, as dense vector matching struggles to precisely associate a specific entity in the user query with abstract concepts on top of the tree. 

\begin{table}[t]
\caption{Comparison of RAG frameworks. ``\!\!{\small \textcolor{Dandelion}{\faQuestion}}\!\!'' indicates limited support by specific methods.}
\label{tab:ragcomparison}
\begin{center}
\begin{small}
\adjustbox{max width=1.0\columnwidth}{
\begin{tabular}{lccccc}
\toprule[1.2pt]
Task type & \makecell[c]{Single-\\document} & \makecell[c]{Cross-\\document} & \makecell[c]{Token-\\level} & \makecell[c]{Passage-\\level} & \makecell[c]{Document-\\level} \\
\midrule
Traditional RAG & \textcolor{Green}{\faCheck} & \textcolor{OrangeRed}{\faTimes} & \textcolor{Green}{\faCheck} & \textcolor{Dandelion}{\faQuestion} & \textcolor{OrangeRed}{\faTimes} \\
Graph-RAG & \textcolor{Green}{\faCheck} & \textcolor{Green}{\faCheck} & \textcolor{Green}{\faCheck} & \textcolor{Green}{\faCheck} & \textcolor{Dandelion}{\faQuestion} \\
Tree-RAG & \textcolor{Green}{\faCheck} & \textcolor{OrangeRed}{\faTimes} & \textcolor{Dandelion}{\faQuestion} & \textcolor{Green}{\faCheck} & \textcolor{Green}{\faCheck} \\
\rowcolor{gray!25}
\textbf{$\Psi$-RAG (Ours)} & \textcolor{Green}{\faCheck} & \textcolor{Green}{\faCheck} & \textcolor{Green}{\faCheck} & \textcolor{Green}{\faCheck} & \textcolor{Green}{\faCheck} \\
\bottomrule[1.2pt]
\end{tabular}
}
\end{small}
\end{center}
\vskip -0.2in
\end{table}

This work aims to address the limitations of Tree-RAG on cross-document and multi-hop retrieval. 
We present $\Psi$-RAG, a versatile and effective Tree-RAG framework, comprising a {\it hierarchical abstract tree index} and a {\it multi-granular agentic retriever}. (1) To address poor distribution adaptability, we build an abstract tree index via a hierarchical clustering-inspired ``merging and collapse'' process,  
which iteratively links similar chunks to existing or newly generated upper-level nodes in the tree. 
(2) To tackle structural isolation, a retrieval and answering agent reasons through the user query, during which it calls for additional document retrieval with a reorganized query when needed. This empowers the tree retriever with better causal understanding of user query. 
(3) To resolve the conflict between coarse abstraction and detailed factual retrieval, a multi-granular hybrid retrieval framework is employed, where an agent-powered sparse retriever supplements the hierarchical tree retrieval with fine-grained information, yielding significant performance gains. Our proposed framework possesses numerous advantages:

\begin{itemize}[leftmargin=10pt,itemsep=1pt,topsep=2pt]
\item {\bf $\Psi$-RAG generalizes Tree-RAG from passage-level to corpus-level indexing.} As shown in \cref{fig:corpusindex}(c), $\Psi$-RAG builds a corpus-level tree 10$\times$ faster than OpenIE-based Graph-RAG. Under this setting, $\Psi$-RAG achieves average gains of 23.7\% in retrieval and 25.9\% in generation over RAPTOR on token-level QA.
\item {\bf $\Psi$-RAG extends Tree-RAG to complex multi-hop scenarios.} To our knowledge, $\Psi$-RAG is the first Tree-RAG for cross-document multi-hop scenarios comparable to Graph-RAG, exhibiting great application potential. As in \cref{fig:corpusindex}(c), $\Psi$-RAG 
is both efficient and powerful, outperforming cutting-edge Graph-RAG frameworks like HippoRAG 2 \cite{hipporag2}. 
\item As in \cref{tab:ragcomparison}, {\bf $\Psi$-RAG is an all-in-one Tree-RAG framework supporting tasks across various granularities}: token-level factual QA, passage-level causal reasoning, and document-level summarization.
\item {\bf $\Psi$-RAG is built entirely with open-source LLMs}. Its components are flexibly replaceable. $\Psi$-RAG can be applied to custom corpora without any training or fine-tuning, demonstrating strong generalizability.
\end{itemize}

\begin{figure*}[ht]
  \begin{center}
    \centerline{\includegraphics[width=0.98\textwidth]{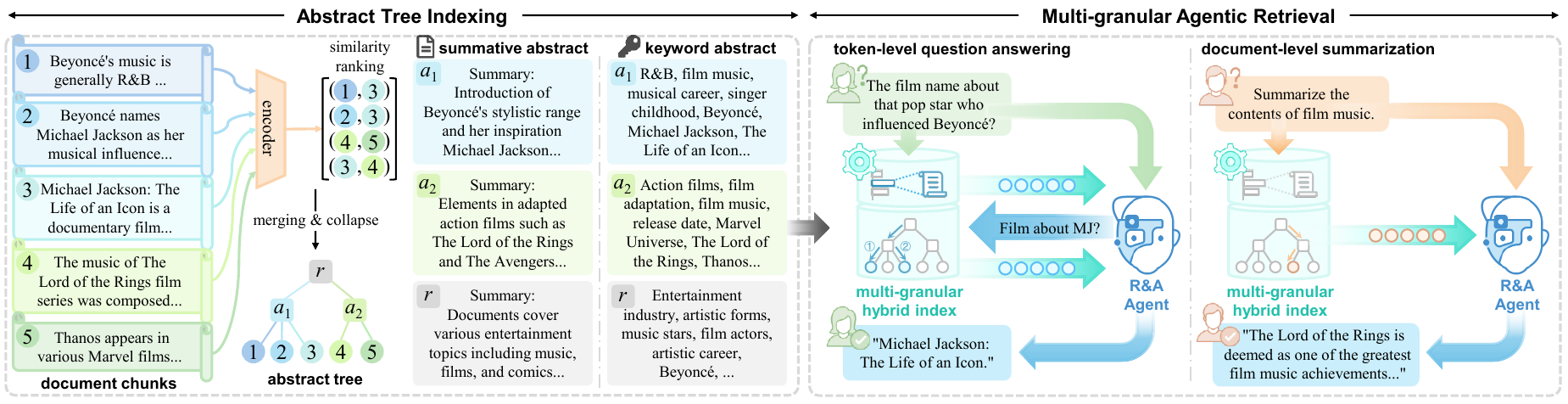}}
    \caption{
      Overview of $\Psi$-RAG. 
      In Abstract Tree Indexing, a hierarchical abstract tree is built by ranking pairwise chunk similarities and applying an iterative ``merging and collapse'' process. 
      Then, an abstraction agent generates a summative or keyword abstract for each abstract node. 
      In Multi-granular Agentic Retrieval, a retrieval and answering (R\&A) agent dynamically interacts with a multi-granular knowledge base organized by both the tree index and a sparse keyword index. 
      By adjusting the retrieval iterations and the utilization of the sparse index, $\Psi$-RAG adapts to diverse user requests from token-level question answering to document-level summarization.
    }
    \label{fig:overview}
  \end{center}
  \vskip -0.25in
\end{figure*}

\section{Preliminary} \label{preliminary}

{\bf Retrieval-Augmented Generation.} \label{preliminaryrag}
Given a user query $q$, an LLM $f(q; I)$ generates a corresponding response $a$, where $I$ is the system instruction for a specific task. For a reasoning model, $a = R \cup y$, where $R$ is the Chain-of-Thought \cite{cot} reasoning text and $y$ is the final answer. An LLM equipped with RAG can be defined as $f \cdot r$, where the retriever $r(q; \mathcal{I})$ maps $q$ to top-$k$ relevant document chunks $\mathcal{D}^* = \{u_i\}_{i=1}^{k} \subset \mathcal{D}$ using a scoring function $s(q, u)$. The external corpus $\mathcal{D}$ contains $n = \vert \mathcal{D} \vert$ chunks in total. 
The index $\mathcal{I}: \mathcal{D} \rightarrow \mathcal{T}$ maps each chunk in $\mathcal{D}$ to a discrete index space $\mathcal{T}$ for better retrieval. 

{\bf Tree Index.} \label{preliminarytree}
We overload the symbol $\mathcal{T}$ to denote a tree. Formally, a tree $\mathcal{T} = (\mathcal{V}, \mathcal{E})$ is a connected acyclic graph with a finite node set $\mathcal{V}=\{v_i\}_{i=1}^{\vert \mathcal{V} \vert}$ and an edge set $\mathcal{E} \subset \mathcal{V} \times \mathcal{V}$.
The root node for a node $u$ is denoted as $root(u)$. The set of children of $u$ is $c(u)=\{v\vert u \rightarrow v \in \mathcal{E}\}$. Let $\ell(u)$ be the set of all leaf nodes reachable from $u$. The depth of a node $\delta(\cdot)$ is defined as the length of the path to the root, with $\delta(root) = 0$. The tree depth is $\delta(\mathcal{T}) = \max_{u \in \mathcal{T}} \delta(u)$. We denote the parent (ancestor) of leaf $u$ at a distance $k$ as $p^{k}(u)$, where $p^0(u) = u$ and $root(u) = p^{\delta(u)}(u)$. 
The least common ancestor (LCA) of two nodes is defined as the deepest node that is an ancestor to both:
\begin{equation}
\text{LCA}(u,v) = \arg \max_{a\in \{p^k(u)\}_{k=1}^{\delta(u)} \cap \{p^k(v)\}_{k=1}^{\delta(v)}} {\delta(a)}
\end{equation}
\vskip -0.15in
A tree is binary (or bifurcating) if all its non-leaf nodes have exactly two children. A tree is multifurcating if all non-leaf nodes have at least two children and at least one node has more than two children. A tree is equi-depth if every leaf is equidistant from the root. 

\section{$\Psi$-RAG}

As illustrated in \cref{fig:overview}, $\Psi$-RAG follows the two-phase paradigm of existing structured RAG: indexing and retrieval. This section formally presents its two core components: Abstract Tree Indexing and Multi-granular Agentic Retrieval.

\subsection{Abstract Tree Indexing} \label{treeindex}

$\Psi$-RAG constructs a tree index from a long document or a document collection, which we term an {\it abstract tree}. The leaves consist of the original document chunks $\mathcal{D}$ and their encoded dense vectors. Upper-level nodes are called {\it abstract nodes}, each containing an abstract of its children generated by an abstraction agent. The abstract tree is constructed via a ``similarity ranking $\rightarrow$ iterative merging and collapse $\rightarrow$ abstraction'' pipeline.

\textbf{Similarity ranking.} First, all document chunks are encoded into fixed-length dense vectors $e(\mathcal{D})\in\mathbb{R}^{\vert \mathcal{D} \vert\times{d_e}}$ using a pre-trained language encoder. A symmetric similarity matrix $\mathbf{S} \in \mathbb{R}^{\vert \mathcal{D} \vert \times \vert \mathcal{D} \vert} = e(\mathcal{D})e(\mathcal{D})^\top$ is then constructed. These chunk pairs are finally sorted in descending order based on their similarity scores.

\begin{figure}[ht]
  \begin{center}
    \centerline{\includegraphics[width=\columnwidth]{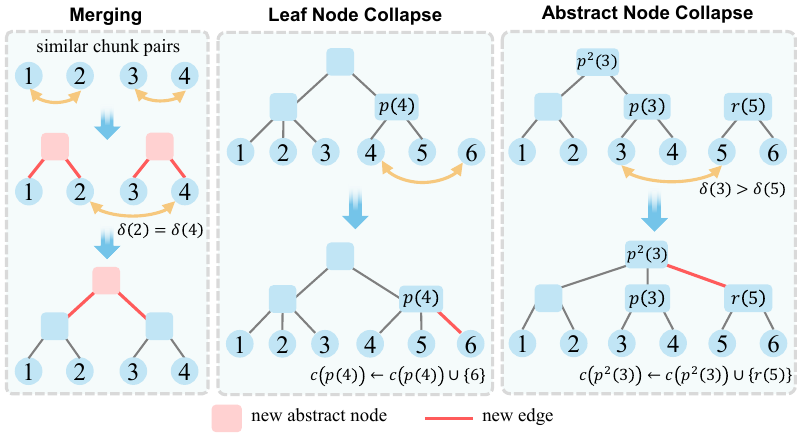}}
    \caption{
      Illustration of abstract tree building steps: merging, leaf node collapse, and abstract node collapse. Yellow arrows indicate the current chunk pair to be processed.
    }
    \label{fig:treebuild}
  \end{center}
  \vskip -0.2in
\end{figure}

\textbf{Iterative merging and collapse.} 
Beginning from the most similar pair of document chunks $(u, v)$, the tree is built iteratively: 
if neither $u$ nor $v$ has a parent, they are linked under a newly created abstract node $a$: $c(a) = \{u, v\}$. This is called a \textit{merging}.
If one node $u$ is already linked to an abstract node $p(u)$, the isolated node $v$ is also linked under $p(u)$: $c(p(u)) \leftarrow c(p(u))\cup \{v\}$. This is called a \textit{leaf node collapse}.
If both $u$ and $v$ are already linked to abstract nodes and their roots differ, then:
\begin{itemize}[leftmargin=10pt,itemsep=1pt,topsep=3pt]
% \item If they share the same root $root(u) = root(v)$, they are already merged and no action is needed.
\item If $\delta(u) = \delta(v)$, their roots are merged under a new abstract node $a$: $c(a) = \{root(u), root(v)\}$, making $root(u) = root(v) = a$.
\item If $\delta(u) > \delta(v)$, $root(v)$ is linked onto the path of $u$:
\begin{equation}
c(p^{\delta(v)+1}(u)) \leftarrow c(p^{\delta(v)+1}(u)) \cup \{root(v)\}
\end{equation}
This results in $\delta(u)=\delta(v)$. The case for $\delta(u) < \delta(v)$ is analogous. This is called an \textit{abstract node collapse}.
\end{itemize}

\cref{fig:treebuild} illustrates these tree-building steps. The iteration runs for exactly $n-1$ steps, which is the minimum number of connections required to link all $n$ nodes into a single tree.
Subsequently, the abstract tree is {\it rebalanced} by splitting nodes with an excessive number of children and distributing the children evenly between two new nodes. This avoids overly long contexts for the abstraction agent.
A detailed algorithm is provided in Appendix \ref{alg}. 

\textbf{Abstraction.} This step employs an abstraction agent to generate and encode an abstract for each abstract node. We consider two abstraction types. (1) \textit{Summative abstract}: paragraphs of summary text \cite{raptor, graphrag} that capture coherent logic and key entity relationships. (2) \textit{Keyword abstract}: keywords of different types distilled from the document chunks, with high information density but weak logical correlations. Phrases are encouraged over single words to preserve the entity co-occurrence. 
% Unlike traditional keyword generation methods such as named entity recognition, 
% Our keyword abstract consists of three types: 
% \begin{itemize}[leftmargin=10pt,itemsep=1pt,topsep=3pt]
% \item \textit{Taxonomic keywords}: keywords that reflect common topics, such as ``R\&B music'' ``American pop singer'' ``entertainment industry''. These keywords effectively unify documents from diverse sources.
% \item \textit{Thematic keywords}: keywords that describe information types, such as ``artistic career'' ``birthplace'' ``singer awards''. 
% These keywords help retrievers prefer documents with thematic relevance: 
% % These keywords counteract the ``taxonomic relation bias'' where retrievers often prefer documents related to specific entities but lacking thematic relevance \cite{semantic}. 
% chunks of Beyonc\'e's discography are likely unhelpful for a user querying ``Beyonc\'e's birthplace''.
% \item \textit{Terminological keywords}: keywords that distinguish the specific content of each chunk, such as ``Beyonc\'e'' ``Michael Jackson'' ``The Lord of the Rings''. These keywords refine the coarse-grained high-level abstracts.
% \end{itemize}

\subsection{Multi-granular Agentic Retrieval} \label{agenticretrieval}

Existing Tree-RAG frameworks struggle with multi-hop questions due to the structural isolation of tree indexes.
For a multi-hop query like ``{\it Who is the wife of the man} (David Gest) {\it who produced the documentary of the pop star who influenced Beyonc\'e?}'', 
a retriever relying solely on dense vector similarity would prioritize content related to ``Beyonc\'e'' and ``documentary'', while missing the true subject ``David Gest''. % who produced the documentary about Michael Jackson. 
% Furthermore, such implicit associations often follow a causal order: one must first obtain the answer ``{\it The man who produced the documentary is David Gest}'' before knowing ``{\it Who is the wife of David Gest}''. 
% Graph-RAG frameworks naturally address this by capturing causally ordered multi-hop relations with node walks, albeit at the cost of a much slower indexing process to extract them from the corpus. 
As such implicit associations often follow a causal order, existing tree retrieval methods like beam search \cite{raptor, lattice} are insufficient for ordered causality reasoning. 

{\bf Retrieval and Answering (R\&A) Agent. } To capture ordered causality during reasoning, we employ a versatile R\&A agent that intelligently engages in multi-turn interactions with the abstract tree index. 
It first assesses whether the available information is sufficient to answer the user query accurately or whether an additional retriever call is necessary. The agent response is a triple $a=(R, \texttt{<action>}, \cdot)$ where the action tag is either $\texttt{<answer>}$ or $\texttt{<retrieve>}$. If $\texttt{<retrieve>}$ is returned, the agent proposes a new query $q'$ and requests another retrieval. 
The newly retrieved documents are then fed back to the agent alongside the original user query for the next reasoning step.

For each retrieval step, $\Psi$-RAG performs top-down dense vector matching following \cite{raptor}. Starting from the root, it computes the cosine similarity $s(q,u)$ between the query vector and each candidate node at layer $l$, selecting the top-$k$ nodes $\mathcal{U}_l$. Then, all children of the selected nodes $\bigcup_{u \in \mathcal{U}_l} c(u)$ become the candidates for the next layer. This process iterates until the leaf layer is reached.

Our agentic retrieval workflow is as follows. Given a user query $q$, an initial retrieval is invoked. The returned top-$k$ chunks $\mathcal{D}^*_0$ are fed to the agent, yielding $a_1 = f(q_0; I(\mathcal{D}^*_0))$. Then, ``retrieval-generation'' steps are iterated:
\vskip -0.4in
\begin{equation}
\!\!\mathcal{D}^*_i\!=\!r(q'_{i},\!\mathcal{T}), a_{i+1}\!=\!f\big(q;\bigcup_{j=1}^{i}{\!(I(\mathcal{D}^*_j)\!\cup\!a_j)} \cup I(\mathcal{D}^*_0)\big) \label{eq:iterretrieve}
\end{equation}
\vskip -0.1in
until the action tag becomes $\texttt{<answer>}$ or the number of retrieval steps reaches a threshold $i_{\max}$.
The user instruction $I$ continuously reminds the agent of the remaining retrieval attempts. 
If the information remains insufficient after exhausting all attempts, the R\&A agent returns ``Not mentioned''.

{\bf Agent-powered Hybrid Retrieval.}
In a corpus-level search space, the coarse-grained upper-level abstracts make even simple factual searches challenging for Tree-RAG. 
To better capture fine-grained information, a sparse keyword index is constructed during tree indexing using BM25 \cite{bm25}. During retrieval, the R\&A agent combines results from both the tree and sparse indexes via two optional strategies: (1) parametrically, using a reranker model to integrate both top-$k$ chunks; or (2) non-parametrically, via reciprocal rank fusion (RRF) \cite{rrf}. 

Furthermore, the R\&A agent broadens the scope of sparse retrieval. It is encouraged to reorganize its query with enriched context, such as inserting descriptive appositions: 

$q:$ {\it ``Who is the wife of David Gest?''} $\xrightarrow{\text{reorganization}}$ $q'$: {\it ``Who is the wife of \textbf{the American film producer} David Gest?''}

This adds thematic keywords for the sparse retriever and, more importantly, aids the tree retriever in locating relevant upper-level abstract nodes with the high-level context. 
An algorithm of the agentic retrieval is given in Appendix \ref{alg}.

\section{Distribution Adaptability of Tree-RAG}

This section analyzes the distribution adaptability of $k$-means-type Tree-RAG and $\Psi$-RAG.
Tree-RAG methods like RAPTOR \cite{raptor} build an equi-depth tree index via bottom-up $k$-means-type clustering, such as the Gaussian Mixture Model (GMM). However, $k$-means-type clustering tends to assign objects from major clusters to minor ones, resulting in clusters of relatively uniform size even when the corpus follows a skewed distribution. This is known as the ``uniform effect'' \cite{kmeansuniform06, kmeansuniform12, kmeansuniform23}. 
For RAG, this effect may cause the retriever to overlook chunks with minor themes. 
In contrast, $\Psi$-RAG better preserves skewed distributions in corpora. We compare the two tree RAG frameworks through the lens of their optimization objectives.

\subsection{The Uniform Effect}

{\bf $k$-means-type tree index.} 
% Define an equivalence relation "cluster" on the set of document chunks $\ell(\mathcal{T})$:
% \
% \begin{equation}
%     \rho = \big\{(a,b)\in \ell(\mathcal{T}) \times \ell(\mathcal{T}) | a, b \in \ell(\mathcal{T}_c) \ \forall \mathcal{T}_c \subseteq\mathcal{T}, \mathcal{V}(\mathcal{T}_c) = \bigcup_{u \in \ell(\mathcal{T_c})}{\{p^k(u)\}_{k=0}^{\delta(\mathcal{T}_c)}}\big\}
% \end{equation}
% It partitions the tree into subtrees ${\mathcal{T}_c}$ and correspondingly partitions $\ell(\mathcal{T})$ into clusters $c=\ell(\mathcal{T}_c)$. 
A cluster can be defined as a family of leaf sets $\mathcal{C}= \{c=\ell(\mathcal{T}_c)\}$ from an equi-depth forest $\{\mathcal{T}_c\subseteq\mathcal{T}|\delta(\mathcal{T}_c)=\delta(\mathcal{T})-1\}$. 
A $k$-means-type cluster tree searches for an optimal $\mathcal{C}$ by minimizing the objective
\vskip -0.15in
\begin{equation}
\mathcal{L} = \sum_{i=1}^{k} {\sum_{x_i \in c_i}{\Vert x_i - \mu_i \Vert}^2}, \label{kmeansloss}
\end{equation}
\vskip -0.1in
where $k=\vert \mathcal{C}\vert$, $\mu_i = \sum_{x_i \in c_i}{\frac{x_i}{n_i}}$ is the cluster centroid and $n_i = \vert c_i \vert$. \cref{thm:1} presents a variant of \eqref{kmeansloss} that reveals the ``uniform effect'' in $k$-means clustering. 

\begin{theorem}[Uniform effect of $k$-means \cite{kmeansuniform12}]
  \label{thm:1}
Given $\mathcal{L}$ in \eqref{kmeansloss}, we have
\vskip -0.1in
\begin{equation}
\mathcal{L} = - A\sum_{1 \le i < j \le k}{n_in_j {\Vert \mu_i - \mu_j \Vert}^2} + B \label{uniformeffect}
\end{equation}
\vskip -0.1in
where $A$ and $B$ are constants dependent only on the data. 
\end{theorem}

See \cref{proof1} for the proof. 
Assuming $n_in_j$ and ${\Vert \mu_i - \mu_j \Vert}^2$ are independent, maximizing the term $n_in_j$ in \eqref{uniformeffect} leads to $n_1=n_2=\cdots=n_k=n/k$. This indicates that $k$-means favors clusters of uniform size. 
% The phenomenon has been studied extensively including the analyses under non-independent assumptions; see \cite{kmeansuniform12} and relative literature. 
(Proposition 1, \cite{kmeansuniform23}) further notes that soft $k$-means, {\it i.e.}, GMM, shares the same global optimum and the uniform effect with $k$-means. 
\cref{fig:uniformeffect} shows a toy example where $k$-means gradually shifts the minor cluster centroid toward the major class, inflating the minor cluster beyond its actual size. 
% Given the prevalence of $k$-means, we believe this is also a potential issue for other hierarchical indexes such as Faiss \cite{faiss}. 
This may confuse abstract nodes in a tree index which leads to semantic bias or neglect of minor nodes.

\begin{figure}[ht]
  \begin{center}
    \centerline{\includegraphics[width=\columnwidth]{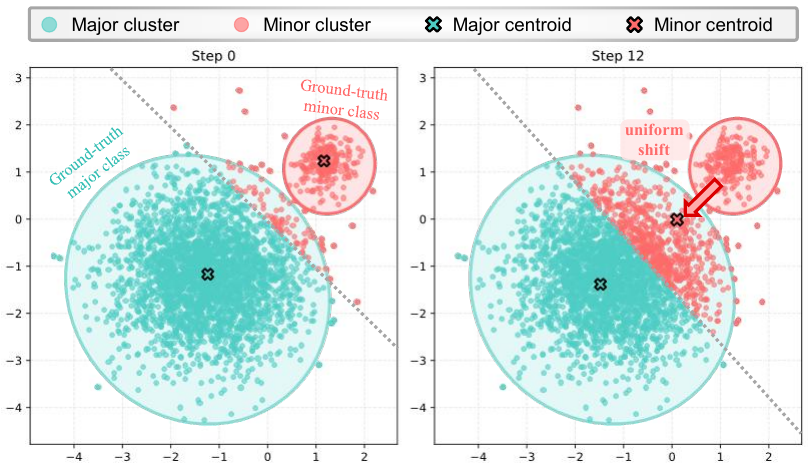}}
    \caption{
      The uniform effect of $k$-means clustering. Initial centroids are set at the center of each class (step 0). 
      % At step 0, only a few major class nodes are misassigned to the minor cluster. 
      As iteration proceeds, more major class nodes are misassigned to the minor cluster, eventually shifting the minor centroid (step 12).
    }
    \label{fig:uniformeffect}
  \end{center}
  \vskip -0.2in
\end{figure}

{\bf $\Psi$-RAG's hierarchical index.} 
$\Psi$-RAG is inspired by agglomerative hierarchical clustering (AHC) \cite{singlelinkage, upgma}, a greedy heuristic algorithm without an explicit optimization objective. We therefore adopt a widely used metric for evaluating AHC.

\begin{definition}[Dasgupta's cost \cite{dasgupta}]
\label{def:dc}
For a tree $\mathcal{T}$, the Dasgupta's cost is defined as
\begin{equation}
c(\mathcal{T})=\sum_{1\le i < j \le n}d(u_i,u_j) \vert \ell(\text{LCA}(u_i,u_j)) \vert
\end{equation}
where $d(\cdot, \cdot)$ is a distance function. 
\end{definition}

The intuition is that semantically similar points (with small $d$) should have a deeper LCA, {\it i.e.,} be placed closer in the tree. 
% meaning the size of their ancestor's subtree should be relatively small, bringing them closer in the index space. 
% If two similar points have an LCA near the root, they are far apart in the index space, incurring a high cost. 
% Directly minimizing this cost is an NP-hard problem, so it cannot be proven to be the exact implicit objective of AHC. 
Dasgupta’s cost is widely regarded as the implicit minimization objective of AHC \cite{dccite1, dccite2, dccite3}.
% This is further supported by the relationship between approximation algorithms such as the sparsest cut \cite{dasgupta}.

\begin{figure*}[ht]
  \begin{center}
    \centerline{\includegraphics[width=0.98\textwidth]{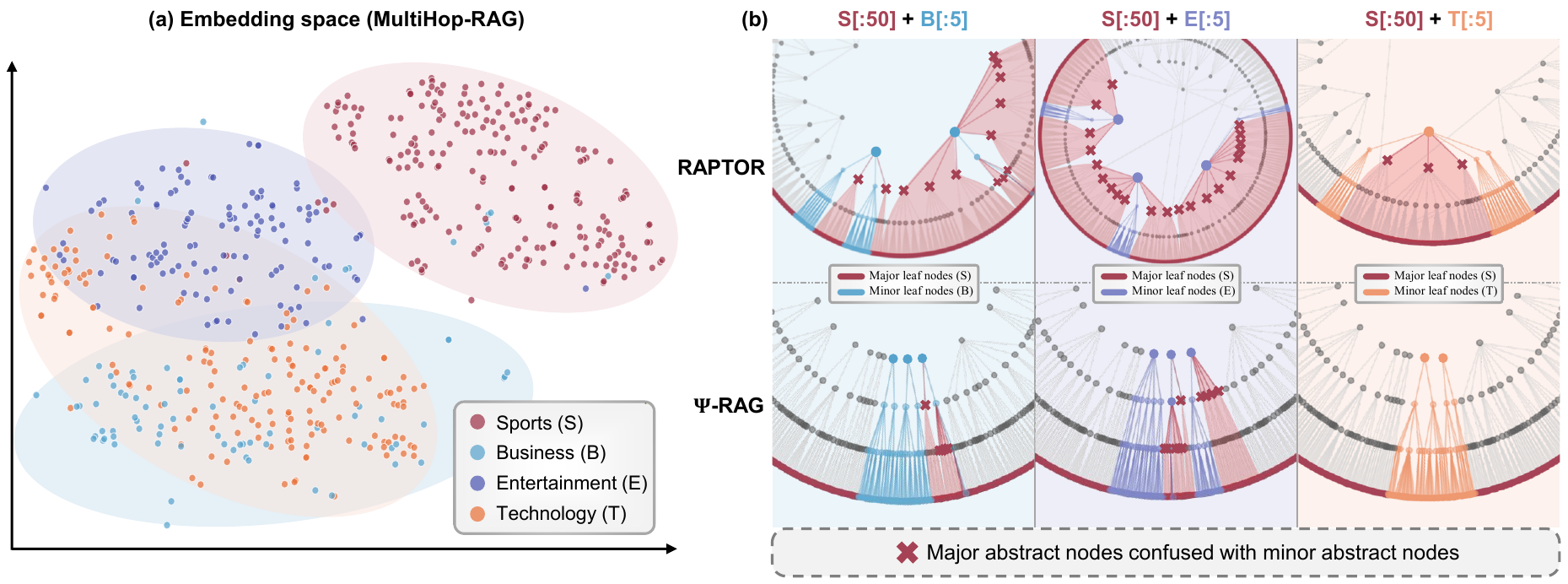}}
    \caption{
RAPTOR vs. $\Psi$-RAG on skewed distributions.
(a) 2-dimensional embedding space of the four main categories of MultiHop-RAG.
(b) Circular tree visualizations (only regions containing minor class nodes are shown). The outermost ring represents leaf nodes. For \textcolor{BrickRed}{S[:50]}+\textcolor{CornflowerBlue}{B[:5]}, \textcolor{BrickRed}{red} points represent ``\textcolor{BrickRed}{Sports}'' nodes, while \textcolor{CornflowerBlue}{blue} points and edges trace root-to-leaf paths for ``\textcolor{CornflowerBlue}{Business}'' nodes. The red shaded area indicates the retrieval noise introduced by confused nodes (\!\!{\small  \textcolor{BrickRed}{\faTimes}}\!\!). Comparing the number of confused nodes and the size of the shaded areas, $\Psi$-RAG organizes minor classes more effectively.
    }
    \label{fig:distributionadapt}
  \end{center}
\vskip -0.3in
\end{figure*}

Based on Dasgupta's cost, we analyze two scenarios: (1) whether $\Psi$-RAG tends to build a tree with uniform leaf distribution, and (2) whether it adaptively generates subtrees of varying sizes for minor clusters. Our findings are formalized in the following Theorems.

\begin{theorem}[Non-uniformity preference of $\Psi$-RAG] \label{thm:2}
Consider a 3-layer equi-depth multifurcating tree $\mathcal{T}$ with a perfectly uniform leaf distribution among its subtrees $\mathcal{T}_1, \cdots, \mathcal{T}_k$: $n_1=n_2=\cdots=n_k=n/k$, where $k \ge 3$, $n \ge 3k$. A new tree $\mathcal{T}'$ is constructed by moving one leaf from $\mathcal{T}_j$ to another subtree $\mathcal{T}_i$. Then: \\ %  such that $n'_i=n/k+1, n'_j=n/k-1$
(1) $c(\mathcal{T}') < c(\mathcal{T})$; \\
(2) Let $\mathcal{T} = \psi(\mathcal{T}_{\text{AHC}})$ and $\mathcal{T'}=\psi(\mathcal{T}'_{\text{AHC}})$, where $\psi(\cdot)$ denotes the collapsing process from a binary tree to an equi-depth multifurcating tree. Denote by $\Delta c$ and $\Delta c'$ the increments of Dasgupta's cost on $\mathcal{T}$ and $\mathcal{T}'$ introduced by $\psi(\cdot)$, respectively. Then we have $c(\mathcal{T}') - c(\mathcal{T}) = \Delta c' - \Delta c$.
\end{theorem}

\begin{theorem}[Distribution adaptability of $\Psi$-RAG] \label{thm:3}
Consider a 3-layer equi-depth multifurcating tree $\mathcal{T}$ with subtrees $\mathcal{T}_1, \cdots, \mathcal{T}_k$ of any size where $k \ge 3$. A subtree of a minor leaf cluster $\mathcal{T}_i$ has $2 \le n_i < n_j, \forall 1 \le j \ne i \le k$. A new tree $\mathcal{T}'$ is constructed by moving one leaf from some subtree $\mathcal{T}_j$ to the minor subtree $\mathcal{T}_i$. Then: \\
(1) $c(\mathcal{T}') \ge c(\mathcal{T})$, and $c(\mathcal{T}') = c(\mathcal{T})$ iff $n_j=n_i + 1$; \\
(2) Let $\mathcal{T} = \psi(\mathcal{T}_{\text{AHC}})$,  $\mathcal{T'}=\psi(\mathcal{T}'_{\text{AHC}})$ and $\Delta c, \Delta c'$ defined as in \cref{thm:2}. Then we have $\Delta c' \ge \Delta c$, and $\Delta c' = \Delta c$ iff $n_j=n_i + 1$.
\end{theorem}
See \cref{proof2} and \cref{proof3} for the proofs and remarks. The above theorems prove that $\Psi$-RAG incurs a larger increment in Dasgupta's cost for (1) a tree with perfectly uniform leaf distribution, and (2) a tree where nodes from major clusters are assigned to a minor cluster. They show that $\Psi$-RAG does not suffer from the uniform effect and, furthermore, adaptively preserves skewed distributions within corpora.

\subsection{Visualization Study on Skewed Data} \label{skewvisual}

We conduct a visualization study by constructing tree indexes with both RAPTOR and $\Psi$-RAG on skewed datasets to validate the above conclusions. Three skewed datasets are generated from MultiHop-RAG \cite{multihoprag}: each combines the first 5 documents from one of three categories: ``Business'', ``Entertainment'', or ``Technology'', with the  first 50 documents from the semantically distant ``Sports'' class, as shown in \cref{fig:distributionadapt}(a). We extract the bottom four layers of each tree 
% (as upper layers near the root are too coarse) 
and visualize them as circular trees. 

The results in \cref{fig:distributionadapt}(b) reveal that major abstract nodes in RAPTOR tend to be confused with ({\it i.e.}, share a parent with) minor ones especially at higher layers, which is a manifestation of the uniform effect. 
In contrast, $\Psi$-RAG assigns far fewer major abstract nodes to minor ones, resulting in an intuitively superior index structure. A detailed case study of the confused abstracts and its negative effect on retrieval is provided in Appendix \ref{treeablation}.

\section{Experiments}

\subsection{Experimental Setup}

{\bf Tasks and evaluation metrics.} Our evaluation encompasses four types of tasks.
(1) Single-hop and multi-hop QA (token-level fact answering) involves detailed facts within one or multiple passages under a {\it cross-document setting}. A word or a phrase is required as an answer. We use Recall@$k$ ($k=2,5$) for retrieval metrics and use Exact Match (EM) and F1 score for QA. 
(2) Narrative QA (passage-level causal reasoning) requires the model to identify causal relationships within a {\it single long document} for a narrative answer. Given the variable answer length, we employ F1 and the Recall-Oriented Understudy for Gisting Evaluation using the Longest Common Subsequence (ROUGE-L) as metrics.
(3) Document-level summarization requires a comprehensive understanding of a {\it single long document} in response to a fixed query: ``Summarize the contents of this meeting / news event''. 
Iterative retrieval and the sparse index are disabled for this task. 
Given the variable answer length and format, we use ROUGE-L as the metric.

\begin{table*}[t]
\caption{Single- and multi-hop QA results in EM (\%) and F1 (\%). Best results are in {\bf bold}. Second-bests are \underline{underlined}. }
\vskip -0.1in
\label{tab:tokenqa}
\begin{center}
\adjustbox{max width=\textwidth}{
\begin{tabular}{clcccccccccccc|cc}
\toprule[1.2pt]
\multirow{2}{*}{Type} & \multirow{2}{*}{Method} & \multicolumn{2}{c}{NQ} & \multicolumn{2}{c}{PopQA} & \multicolumn{2}{c}{HotpotQA} & \multicolumn{2}{c}{2Wiki} & \multicolumn{2}{c}{MuSiQue} & \multicolumn{2}{c}{{MultiHop-RAG}} & \multicolumn{2}{c}{{\it Average}} \\ \cmidrule(){3-16}
& & EM & F1 & EM & F1 & EM & F1 & EM & F1 & EM & F1 & EM & F1 & EM & F1 \\
\midrule
\multirow{4}{*}{\makecell[c]{Traditional\\RAG}} & BM25 & $45.4$ & $56.76$ & $40.9$ & $49.01$ & $43.8$ & $54.07$ & $34.7$ & $36.90$ & $13.0$ & $17.96$ & $2.5$ & $3.93$ & $23.55$ & $36.44$ \\
& DPR & $46.0$ & $59.09$ & $47.2$ & $55.89$ & $52.0$ & $63.10$ & $39.9$ & $43.13$ & $18.7$ & $24.03$ & $41.3$ & $42.72$ & $40.85$ & $47.99$ \\
& BM25 + DPR + Q3R & $48.2$ & $60.58$ & \underline{$47.2$} & $56.01$ & $58.4$ & $70.98$ & $41.6$ & $45.43$ & $28.3$ & $34.81$ & $52.9$ & $53.49$ & $46.10$ & $53.55$ \\ 
& DPR + IRCoT + Q3R & \underline{$50.1$} & $62.71$ & $\mathbf{47.3}$ & $\mathbf{56.98}$ & $59.1$ & $71.00$ & $48.9$ & $53.44$ & $29.7$ & $37.36$ & $41.7$ & $42.28$ & $46.13$ & $53.96$ \\ \midrule
& GraphRAG & $36.0$ & $50.85$ & $32.7$ & $42.13$ & $54.9$ & $66.30$ & $52.2$ & $56.62$ & $25.8$ & $30.70$ & $37.3$ & $39.08$ & $39.82$ & $47.61$ \\
Graph-RAG & HippoRAG 2 & $47.4$ & $62.42$ & $43.4$ & $55.90$ & $61.9$ & $\mathbf{75.40}$ & $64.8$ & $71.35$ & $37.5$ & \underline{$48.00$} & $17.7$ & $19.36$ & $45.45$ & $55.41$ \\
& HippoRAG 2 + Q3R & $45.7$ & $60.99$ & $43.3$ & $56.05$ & $58.7$ & $71.52$ & $60.1$ & $66.00$ & $32.9$ & $43.27$ & $17.6$ & $19.37$ & $43.05$ & $52.87$ \\ \midrule
& RAPTOR + Q3R & $43.5$ & $55.40$ & $19.8$ & $25.76$ & $43.1$ & $52.54$ & $19.1$ & $22.16$ & $16.6$ & $21.10$ & $42.5$ & $43.20$ & $30.77$ & $36.69$ \\
\rowcolor{gray!25}
\cellcolor{white} \makecell[c]{Tree-RAG} & $\Psi$-RAG \!\!{\small \faFile*[regular]}\!\! & $50.0$ & \underline{$63.45$} & $46.7$ & \underline{$56.74$} & \underline{$62.1$} & $74.59$ & $\mathbf{69.1}$ & \underline{$76.73$} & $\mathbf{38.7}$ & $\mathbf{48.90}$ & $\mathbf{55.3}$ & $\mathbf{56.23}$ & $\mathbf{53.65}$ & $\mathbf{62.77}$ \\
\rowcolor{gray!25}
\cellcolor{white} & $\Psi$-RAG \!\!{\small \faKey}\!\! & $\mathbf{50.6}$ & $\mathbf{63.58}$ & $46.1$ & $55.88$ & $\mathbf{62.2}$ & \underline{$74.85$} & \underline{$69.0$} & $\mathbf{76.94}$ & \underline{$37.9$} & $47.83$ & \underline{$54.1$} & \underline{$55.18$} & \underline{$53.32$} & \underline{$62.38$} \\
\bottomrule
\end{tabular}
}
\end{center}
\end{table*}

\begin{table*}[t]
\caption{Retrieval results for single- and multi-hop QA in Recall@2 (R@2\%) and Recall@5 (R@5\%). Best results are in {\bf bold}. Second-bests are \underline{underlined}. }
\vskip -0.1in
\label{tab:tokenrecall}
\begin{center}
\adjustbox{max width=\textwidth}{
\begin{tabular}{clcccccccccc|cc}
\toprule
\multirow{2}{*}{Type} & \multirow{2}{*}{Method} & \multicolumn{2}{c}{NQ} & \multicolumn{2}{c}{PopQA} & \multicolumn{2}{c}{HotpotQA} & \multicolumn{2}{c}{2Wiki} & \multicolumn{2}{c}{MuSiQue} & \multicolumn{2}{c}{{\it Average}} \\ \cmidrule(){3-14}
& & R@2 & R@5 & R@2 & R@5 & R@2 & R@5 & R@2 & R@5 & R@2 & R@5 & R@2 & R@5 \\
\midrule
\multirow{4}{*}{\makecell[c]{Traditional\\RAG}} & BM25 & $28.20$ & $56.03$ & $23.90$ & $35.70$ & $57.25$ & $74.80$ & $54.70$ & $64.98$ & $28.78$ & $35.40$ & $38.57$ & $53.38$ \\
& DPR & $27.13$ & $69.51$ & $20.85$ & $49.70$ & $35.25$ & $86.75$ & $32.30$ & $71.03$ & $33.11$ & $48.08$ & $29.73$ & $65.01$ \\
& BM25 + DPR + Q3R & $\mathbf{46.29}$ & \underline{$76.63$} & $42.90$ & $50.55$ & $85.70$ & $93.80$ & $68.95$ & $74.52$ & $50.92$ & $67.35$ & $58.95$ & $72.57$ \\ 
& DPR + IRCoT + Q3R & \underline{$46.16$} & $76.10$ & $42.75$ & $50.45$ & $84.50$ & $92.40$ & $69.20$ & $74.08$ & $50.90$ & $66.02$ & $58.70$ & $71.81$ \\ \midrule
\multirow{2}{*}{\makecell[c]{Graph-RAG}} & HippoRAG 2 & $45.32$ & $\mathbf{77.11}$ & $44.30$ & \underline{$50.75$} & $84.05$ & $94.40$ & $75.40$ & $89.92$ & $55.67$ & $73.77$ & $60.95$ & $77.19$ \\
& HippoRAG 2 + Q3R & $42.09$ & $73.68$ & $\mathbf{46.60}$ & $50.45$ & $74.70$ & $90.65$ & $70.00$ & $80.73$ & $49.02$ & $64.64$ & $56.48$ & $72.03$ \\
\midrule
& RAPTOR + Q3R & $41.56$ & $68.48$ & $21.20$ & $23.80$ & $67.10$ & $72.25$ & $40.93$ & $44.57$ & $33.29$ & $42.26$ & $40.82$ & $50.27$ \\
\rowcolor{gray!25}
\cellcolor{white} Tree-RAG & $\Psi$-RAG \!\!{\small \faFile*[regular]}\!\! & $46.08$ & $75.72$ & \underline{$43.35$} & $\mathbf{51.30}$ & $\mathbf{88.15}$ & $\mathbf{96.00}$ & $\mathbf{81.35}$ & $\mathbf{96.13}$ & $\mathbf{57.21}$ & $\mathbf{75.77}$ & $\mathbf{63.23}$ & $\mathbf{78.98}$ \\
\rowcolor{gray!25}
\cellcolor{white} & $\Psi$-RAG \!\!{\small \faKey}\!\! & $45.88$ & $75.13$ & $42.60$ & $50.55$ & \underline{$88.10$} & \underline{$95.80$} & \underline{$81.23$} & \underline{$95.85$} & \underline{$55.69$} & $75.08$ & \underline{$62.70$} & \underline{$78.48$} \\
\bottomrule[1.2pt]
\end{tabular}
}
\end{center}
\end{table*}

{\bf Datasets.} Each task is evaluated on multiple benchmark datasets. For single-hop and multi-hop QA, we mainly follow HippoRAG 2 \cite{hipporag2} and use 1,000 randomly sampled queries from (1) single-hop: NQ \cite{nq} and PopQA \cite{popqa}; (2) multi-hop: HotpotQA \cite{hotpotqa}, 2Wiki (2WikiMultihopQA) \cite{2wiki}, and MuSiQue \cite{musique}. We also include the first 1,000 questions from MultiHop-RAG \cite{multihoprag}, a multi-hop benchmark containing negative samples. For narrative QA, 10 long documents 
% and their corresponding 293 queries 
are randomly sampled from NarrativeQA \cite{narrativeqa}. We also use the first 100 documents from the LongBook subset in $\infty$Bench \cite{inftybench}, a narrative QA dataset featuring significantly longer documents. For summarization, we adopt QMSum \cite{qmsum} and WCEP \cite{wcep}. Overly short documents are filtered out following \cite{gor}. Further dataset details are provided in Appendix \ref{setupapp}.

{\bf Baselines.} Our baselines cover various RAG paradigms. (1) Traditional RAG: the sparse retriever BM25 \cite{bm25}, the dense vector retriever DPR \cite{dpr} with Faiss \cite{faiss}, and the strong hybrid retriever BM25 + DPR equipped with a reranker. Another baseline involves DPR equipped with an iterative agentic retriever IRCoT \cite{ircot}. (2) Graph-RAG: GraphRAG from Microsoft
% \footnote{We use Nano-GraphRAG (\url{https://github.com/gusye1234/nano-graphrag}) as the implementation.}
\cite{graphrag}, the state-of-the-art framework HippoRAG 2\footnote{
% As NV-Embed-V2 is not compatible with Q3R, 
We report both the original setting (with NV-Embed-V2 \cite{nvembed}) and our setting with Q3-8B and Q3R.} \cite{hipporag2}, 
and GoR \cite{gor}, a new summarization-exclusive Graph-RAG framework. 
(3) Tree-RAG: RAPTOR \cite{raptor}, the only open-source framework that can be technically implemented to a corpus-level index.

{\bf Implementation details.} 
$\Psi$-RAG (and all dense retrieval baselines) by default uses Qwen3-Embedding-8B (Q3-8B) \cite{qwen3embedding} as the encoder and Llama-3.3-70B (L3-70B) \cite{llama} as the abstraction and R\&A agent. We use BM25S \cite{bm25s} as the BM25 implementation. $\Psi$-RAG is equipped with Qwen3-Reranker-8B (Q3R) to fuse retrieval results from sparse and dense retrievers. For a fair comparison, most baselines are also equipped with Q3R. 
Detailed configurations and instruction prompts are provided in Appendix \ref{setupapp}. 

\subsection{Single- and Multi-Hop QA} \label{tokentask}

We use the symbols ``\!{\small \faFile*[regular]}\!'' and ``\!{\small \faKey}\!'' to denote $\Psi$-RAG with summative and keyword abstract, respectively. Retrieval and QA results on single-hop and multi-hop datasets are presented in \cref{tab:tokenqa} and \ref{tab:tokenrecall}, yielding several key observations.

(1) {\bf Traditional baselines excel at single-hop QA.} For factual questions with little context such as NQ and PopQA, DPR and its hybrid combination with BM25 achieve recall and accuracy competitive with HippoRAG 2 and $\Psi$-RAG. In such cases, answers often lie in documents with frequent query-related entities or strong semantic correlation. 

(2) {\bf Multi-hop QA is the domain of iterative agentic RAG and structured RAG.} Both IRCoT and HippoRAG 2 outperform the hybrid retriever BM25 + DPR on 3 out of 4 multi-hop datasets, with the highest F1 gains exceeding 20\%. This underscores the necessity of modeling implicit entity relationships for multi-hop reasoning. 

(3) {\bf RAPTOR's effectiveness diminishes at corpus scale.} On 5 out of 6 datasets, RAPTOR + Q3R underperforms DPR without a reranker, showing degradations exceeding 25\% in Recall@5 and 30\% in F1 on PopQA. 
This highlights a key limitation of traditional Tree-RAG:
as the index space expands, performance on factual QA is severely hampered by coarse-grained abstracts. 

(4) {\bf $\Psi$-RAG achieves strong performance across both task types.} Compared to RAPTOR, $\Psi$-RAG achieves average gains of 25.9\% in F1 and 23.7\% in Recall@5, irrespective of the abstraction type. Notably, $\Psi$-RAG surpasses RAPTOR by up to 54.78\% in F1 on 2Wiki. To our knowledge, $\Psi$-RAG is the first Tree-RAG to surpass advanced Graph-RAG, outperforming HippoRAG 2 by 7.4\% and HippoRAG 2 + Q3R by 9.9\% in average F1. 
We also find that the summative abstract slightly outperforms the keyword abstract. This indicates that logical relationships between entities can effectively compensate for abstract coarseness. 

\begin{table}[t]
\caption{Narrative QA results in F1 (\%) and ROUGE-L (R-L\%), and summarization results in ROUGE-L (R-L\%). Best results are in {\bf bold}. Second-bests are \underline{underlined}. }
\vskip -0.1in
\label{tab:highlevel}
\begin{center}
\adjustbox{max width=1.0\columnwidth}{
\begin{tabular}{lcccccc}
\toprule[1.2pt]
\multirow{2}{*}{Method} & \multicolumn{2}{c}{NarrativeQA} & \multicolumn{2}{c}{$\infty$-LongBook} & QMSum & WCEP \\ \cmidrule(){2-7}
& F1 & R-L & F1 & R-L & R-L & R-L \\
\midrule
BM25 & $13.95$ & $14.07$ & $14.87$ & $15.54$ & $15.63$ & $16.08$ \\
DPR & $17.78$ & $17.87$ & $21.51$ & $23.60$ & $13.43$ & $8.14$ \\
BM25 + DPR + Q3R & $20.30$ & $20.37$ & $24.46$ & $24.87$ & $13.42$ & $8.14$ \\
\midrule
GoR + Q3R & -- & -- & -- & -- & $15.98$ & $13.98$ \\
\midrule
RAPTOR + Q3R & $21.36$ & $21.61$ & $13.90$ & $15.85$ & $16.53$ & $18.33$ \\
\rowcolor{gray!25}
$\Psi$-RAG \!\!{\small \faFile*[regular]}\!\! & $\mathbf{27.84}$ & $\mathbf{28.43}$ & \underline{$25.52$} & \underline{$27.11$}  & $\mathbf{16.87}$  & $\mathbf{20.28}$  \\
\rowcolor{gray!25}
$\Psi$-RAG \!\!{\small \faKey}\!\! & \underline{$27.66$} & \underline{$28.35$} & $\mathbf{28.44}$ & $\mathbf{28.36}$ & \underline{$16.65$} & \underline{$18.81$} \\
\bottomrule[1.2pt]
\end{tabular}
}
\end{center}
\vskip -0.2in
\end{table}

\subsection{Narrative QA and Summarization} \label{highlevel}

\cref{tab:highlevel} shows the performance of narrative QA and summarization. We derive the following observations:

(1) {\bf Traditional baselines struggle to capture causal relationships and high-level semantics.} Unlike in token-level QA, BM25 and DPR achieve the lowest performance in most cases. This indicates that simple keyword matching and vector similarity  are insufficient for passage-level and document-level user demands. On the contrary, structured indexes of RAPTOR and GoR empower the agent with high-level information and thus higher performance on passage- and document-level tasks.  

(2) {\bf $\Psi$-RAG excels in both narrative and summarization tasks.} It consistently outperforms all baseline methods across both settings, surpassing RAPTOR by an average of 9.7\% in F1 on NarrativeQA and $\infty$-LongBook. It also surpasses GoR by 3.2\% in average ROUGE-L.

\begin{figure}[h]
  \begin{center}
    \centerline{\includegraphics[width=0.8\columnwidth]{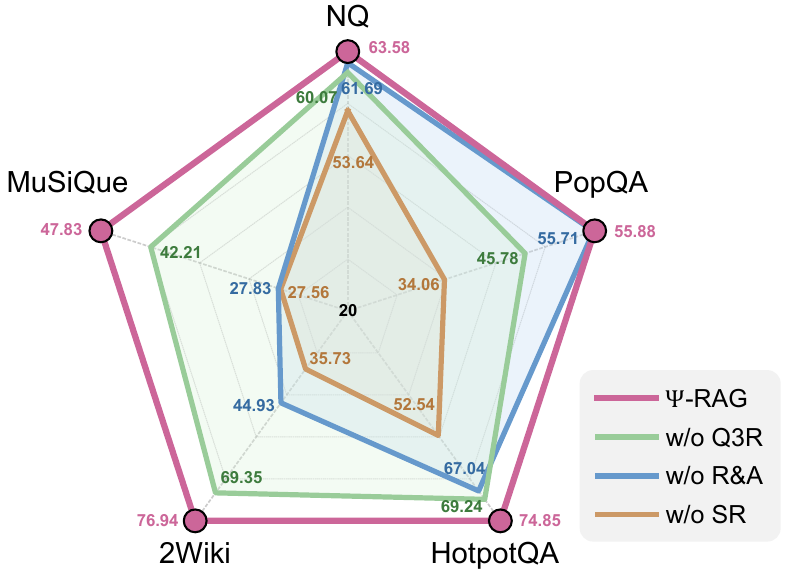}}
    \caption{
      Ablation studies in F1 (\%). 
    }
    \label{fig:ablation}
  \end{center}
  \vskip -0.3in 
\end{figure}

\subsection{Ablation Studies} \label{ablation}

\cref{fig:ablation} presents the QA performance of $\Psi$-RAG with different components ablated: the R\&A agent, sparse retrieval (SR), and reranker (Q3R). We use RRF instead of the reranker in the ``w/o Q3R'' setting. We have some key observations:

(1) {\bf $\Psi$-RAG achieves peak performance when all components are combined together.} Our R\&A agent consistently improves performance on multi-hop QA tasks with an average gain of $\sim$20\%. The contribution of sparse retrieval is also significant, especially on datasets like PopQA (+21.82\%) and 2Wiki (+41.21\%) with large amounts of short factual questions related to specific names. This validates its effectiveness in mitigating the coarse abstraction of the tree index. Furthermore, despite additional computational and storage costs, a large reranker integrates retrievers more effectively than the non-parametric RRF, but it does not provide a core performance contribution in $\Psi$-RAG.

(2) {\bf R\&A agent is also beneficial for single-hop QA. } For example, it brings a 1.89\% F1 gain on NQ. This benefit stems from its query reorganization mechanism, which improves retrieval for queries with insufficient context.

\begin{table}[t]
\caption{Query reorganization (QR) ablation in F1 (\%). Best results are in \textbf{bold}.}
\vskip -0.1in
\label{tab:reorg}
\begin{center}
\begin{small}
\adjustbox{max width=1.0\columnwidth}{
\begin{tabular}{lllll}
\toprule[1.2pt]
& HotpotQA & 2Wiki & MuSiQue & MultiHop-RAG \\
\midrule
\rowcolor{gray!25}
$\Psi$-RAG & $\mathbf{74.85}$ & $\mathbf{76.94}$ & $\mathbf{47.83}$ & $\mathbf{55.18}$ \\
w/o QR & $74.17$ (\textcolor{Red}{$\downarrow 0.68$}) & $76.71$ (\textcolor{Red}{$\downarrow 0.23$}) & $45.82$ (\textcolor{Red}{$\downarrow 2.01$}) & $55.03$ (\textcolor{Red}{$\downarrow 0.15$})\\
\bottomrule[1.2pt]
\end{tabular}
}
\vskip -0.1in
\end{small}
\end{center}
\end{table}

{\bf Query reorganization.} 
\cref{tab:reorg} evaluates a counterpart of $\Psi$-RAG without query reorganization, that is, the R\&A agent is no longer encouraged to reorganize and enrich their generated queries at each retrieval attempt. This leads to performance degradation across multi-hop datasets, with a 2\% F1 drop on MuSiQue. Therefore, the thematic keywords and high-level context introduced by query reorganization contribute to accurate hybrid retrieval of $\Psi$-RAG.

\sethlcolor{gray!25}

\begin{table}[ht]
\caption{Experiment results of $\Psi$-RAG with different open-source backbones. Replaced components are \hl{shaded}. }
\label{tab:local}
\begin{center}
\adjustbox{max width=\linewidth}{
\begin{tabular}{llllcccccc}
\toprule[1.2pt]
\multirow{2.5}{*}{Embedding} & \multirow{2.5}{*}{Abstraction} & \multirow{2.5}{*}{R\&A} & \multirow{2.5}{*}{Reranker} & \multicolumn{2}{c}{HotpotQA} & \multicolumn{2}{c}{2Wiki} & \multicolumn{2}{c}{MuSiQue} \\
\cmidrule(){5-6} \cmidrule(){7-8}  \cmidrule(){9-10} 
& & & & R@5 & F1 & R@5 & F1 & R@5 & F1 \\
\midrule
Q3-8B & L3-70B & L3-70B & Q3R & $95.80$ & $74.85$ & $95.85$ & $76.94$ & $75.08$ & $47.83$ \\
\cellcolor{gray!25}Q3-0.6B & L3-70B & L3-70B & Q3R & $94.40$ & $74.48$ & $95.55$ & $75.94$ & $73.42$ & $46.57$ \\
Q3-8B & \cellcolor{gray!25}L3-1B & L3-70B & Q3R & $94.45$ & $73.23$ & $95.77$ & $76.78$ & $73.14$ & $47.00$ \\
Q3-8B & \cellcolor{gray!25}L3-8B & L3-70B & Q3R & $93.65$ & $73.72$ & $95.83$ & $76.30$ & $74.29$ & $47.46$ \\
Q3-8B & L3-70B & \cellcolor{gray!25}G3-27B & Q3R & $95.00$ & $73.65$ & $94.05$ & $71.78$ & $71.80$ & $49.30$ \\
Q3-8B & L3-70B & L3-70B & \cellcolor{gray!25}BGE & $95.10$ & $74.32$ & $95.80$ & $76.85$ & $72.46$ & $48.40$ \\
\bottomrule[1.2pt]
\end{tabular}
}
\end{center}
\vskip -0.1in
\end{table}

\begin{table}[ht]
\caption{Experiment results of $\Psi$-RAG with different closed-source R\&A backbones.}
\label{tab:api}
\begin{center}
\adjustbox{max width=0.8\linewidth}{
\begin{tabular}{lcccc}
\toprule[1.2pt]
\multirow{2.5}{*}{R\&A} & \multicolumn{2}{c}{HotpotQA} & \multicolumn{2}{c}{MuSiQue} \\ \cmidrule(){2-3} \cmidrule(){4-5} 
& R@5 & F1 & R@5 & F1 \\
\midrule
\smash{\raisebox{-2pt}{\includegraphics[scale=0.12]{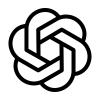}}}~GPT-5-mini & $94.40$ & $78.10$ & $72.67$ & $49.48$ \\
\smash{\raisebox{-2pt}
{\includegraphics[scale=0.02]{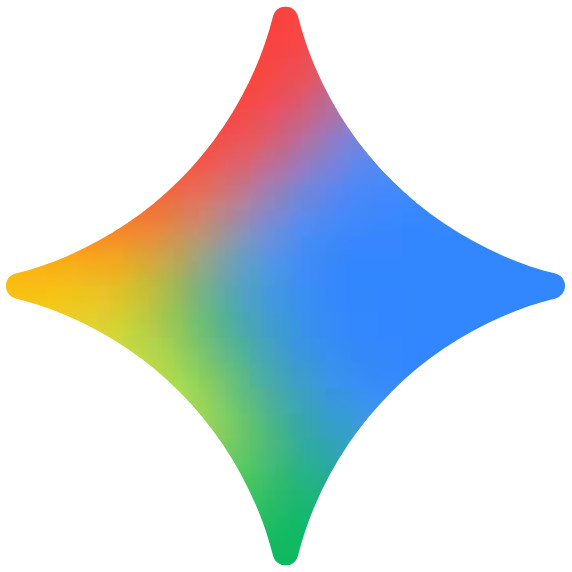}}}~Gemini-2.5-Flash & $96.00$ & $77.18$ & $77.06$ & $51.86$ \\
\bottomrule[1.2pt]
\end{tabular}
}
\end{center}
\vskip -0.1in
\end{table}

\sethlcolor{yellow}
\begin{table*}[t]
\caption{Time costs of different structured RAG. All experiments follow the settings in Appendix \ref{setupapp}. The embedding time of Q3-8B for HotpotQA and MuSiQue is 277.3s and 309.0s, respectively. L3-70B is used by default for abstraction of both RAPTOR and $\Psi$-RAG. ``\textcolor{Green}{$\uparrow$}'' and ``\textcolor{Red}{$\downarrow$}'' indicate a distinguishable higher and lower efficiency compared to RAPTOR + Q3R, respectively.}
\vskip -0.1in
\label{tab:efficiency}
\begin{center}
\adjustbox{max width=1.0\textwidth}{
\begin{tabular}{cllllll}
\toprule[1.2pt]
\multirow{2.5}{*}{Method} & \multicolumn{2}{c}{Indexing time} & \multicolumn{2}{c}{Average single retrieval time} & \multicolumn{2}{c}{Average QA time} \\
\cmidrule(){2-3} \cmidrule(){4-5} \cmidrule(){6-7}
 & \makecell[c]{HotpotQA} & \makecell[c]{MuSiQue} & \makecell[c]{HotpotQA} & \makecell[c]{MuSiQue} & \makecell[c]{HotpotQA} & \makecell[c]{MuSiQue} \\
\midrule
HippoRAG 2 & \makecell[l]{OpenIE: $105,360$s \\ Graph indexing: $32,187$s \\ \textbf{Total}: $137,547$s \textcolor{Red}{$\downarrow \downarrow$}} & \makecell[l]{OpenIE: $128,192$s \\ Graph indexing: $20,668$s \\ \textbf{Total}: $148,860$s \textcolor{Red}{$\downarrow \downarrow$}} & Graph retrieval: $10.44$s \textcolor{Red}{$\downarrow \downarrow$} & Graph retrieval: $8.94$s \textcolor{Red}{$\downarrow \downarrow$} & $19.24$s \textcolor{Red}{$\downarrow$} & $17.56$s \textcolor{Red}{$\downarrow$} \\ \midrule
RAPTOR + Q3R & \makecell[l]{Tree indexing: $1,123$s \\ Abstraction: $5,509$s \\ \textbf{Total}: $6,632$s} & \makecell[l]{Tree indexing: $1,684$s \\ Abstraction: $7,570$s \\ \textbf{Total}: $9,254$s} & \makecell[l]{Tree retrieval: $0.22$s \\ Reranking: $0.98$s \\ \textbf{Total}: $1.20$s} & \makecell[l]{Tree retrieval: $0.22$s \\ Reranking: $2.17$s \\ \textbf{Total}: $2.39$s} & Single retrieval: $8.03$s & Single retrieval: $12.71$s \\ \midrule
\makecell[c]{$\Psi$-RAG} & \makecell[l]{Tree indexing: $124$s \textcolor{Green}{$\uparrow$} \\ Abstraction: $10,425$s \textcolor{Red}{$\downarrow$} \\ \textbf{Total}: $10,549$s \textcolor{Red}{$\downarrow$}} & \makecell[l]{Tree indexing: $258$s \textcolor{Green}{$\uparrow$} \\ Abstraction: $11,899$s \textcolor{Red}{$\downarrow$} \\ \textbf{Total}: $12,157$s \textcolor{Red}{$\downarrow$}} & \makecell[l]{Tree retrieval: $0.19$s \\ Sparse retrieval: $<0.01$s \\ Reranking: $0.79$s \\ \textbf{Total}: $0.98$s \textcolor{Green}{$\uparrow$}} & \makecell[l]{Tree retrieval: $0.19$s \\ Sparse retrieval: $<0.01$s \\ Reranking: $1.07$s \textcolor{Green}{$\uparrow$} \\ \textbf{Total}: $1.26$s \textcolor{Green}{$\uparrow$}} & Agentic retrieval: $8.49$s & Agentic retrieval: $19.66$s \textcolor{Red}{$\downarrow$} \\ \midrule
\makecell[c]{$\Psi$-RAG\\(Abstraction: L3-8B;\\R\&A Agent: G3-27B;\\Reranker: BGE)} & \makecell[l]{Tree indexing: $128$s \textcolor{Green}{$\uparrow$} \\ Abstraction: $2,058$s \textcolor{Green}{$\uparrow$} \\ \textbf{Total}: $2,186$s \textcolor{Green}{$\uparrow$}} & \makecell[l]{Tree indexing: $265$s \textcolor{Green}{$\uparrow$} \\ Abstraction: $2,222$s \textcolor{Green}{$\uparrow$} \\ \textbf{Total}: $2,487$s \textcolor{Green}{$\uparrow$}} & \makecell[l]{Tree retrieval: $0.19$s \\ Sparse retrieval: $<0.01$s \\ Reranking: $0.56$s \textcolor{Green}{$\uparrow$} \\ \textbf{Total}: $0.75$s \textcolor{Green}{$\uparrow$}} & \makecell[l]{Tree retrieval: $0.19$s \\ Sparse retrieval: $<0.01$s \\ Reranking: $0.37$s \textcolor{Green}{$\uparrow$} \\ \textbf{Total}: $0.56$s \textcolor{Green}{$\uparrow$}} & Agentic retrieval: $5.45$s \textcolor{Green}{$\uparrow$} & Agentic retrieval: $7.17$s \textcolor{Green}{$\uparrow$} \\
\bottomrule[1.2pt]
\end{tabular}
}
\end{center}
\end{table*} 

\subsection{Backbone Flexibility} \label{backbone}

$\Psi$-RAG can flexibly accommodate different backbone language models for its embedding model, abstraction agent, R\&A agent, and reranker. We select various open-source models, including (1) Qwen3-Embedding-0.6B (Q3-0.6B) \cite{qwen3embedding} as the embedding model; (2)  Llama-3.2-1B (L3-1B) and Llama-3.1-8B (L3-8B) \cite{llama} as the abstraction agent; (3) Gemma-3-27B (G3-27B)  \cite{gemma} as the R\&A agent; and (4) BGE-Reranker-Large (BGE) \cite{bge} as the reranker. The results are shown in \cref{tab:local}. We observe that these models, especially L3-1B/8B and BGE, serve as effective lightweight alternatives, as they achieve comparable or even superior performance with lower costs. 

We also employ two popular closed-source LLMs: \smash{\raisebox{-2pt}{\includegraphics[scale=0.1]{fig/icon/gpt.png}}}~GPT-5-Mini \cite{gpt5} and \smash{\raisebox{-2pt}
{\includegraphics[scale=0.016]{fig/icon/gemini.png}}}~Gemini-2.5-Flash \cite{gemini}, as the R\&A agent in $\Psi$-RAG. The results are presented in \cref{tab:api}. 
Even with a much smaller size than L3-70B, both models yield certain improvements in F1 score. 
For instance, employing Gemini-2.5-Flash as the R\&A agent achieves a notably high F1 score of 51.86\%.

\subsection{Efficiency Analysis} \label{efficiency}

As $\Psi$-RAG is based on AHC, the computation, sorting, and storage of the similarity matrix introduce time and space complexities of $O(n^2 \log n)$ and $O(n^2)$, respectively\footnote{We omit the vector dimension $d$ as $d \ll n$ in the real-world scenario.}. 
Similarly, the iterative merging and collapse process takes $O(n^2)$ time. 
However, the highly optimized matrix multiplication and sorting techniques make our proposed framework still efficient for a domain-specific index. 

\cref{tab:efficiency} compares the time efficiency of $\Psi$-RAG with other baseline methods. Our abstract tree indexing is very fast, taking only 258 seconds on MuSiQue with 1.3 million tokens (the similarity ranking process only takes $\sim$30s), which is 6.5$\times$ faster than RAPTOR and considerably more efficient than HippoRAG 2 on domain-specific corpora\footnote{Existing Graph-RAG papers tend not to include OpenIE extraction in the calculation of indexing time. This may lead to a significant inconsistency between the reported time and the actual running time.}. 
However, as $\Psi$-RAG includes more abstract nodes, it takes more time for abstraction. 
For retrieval efficiency, the top-down search algorithm inherently has $O(\log n)$ complexity, resulting in lower average retrieval latency compared to HippoRAG 2. 
While $\Psi$-RAG's iterative agentic retrieval process increases the latency, it is proportional to the complexity of the user demand (one additional retrieval attempt opportunity only brings an increment of 0.46s compared to RAPTOR on HotpotQA), so we believe that it is justified in quality improvement for complex queries.  
% Given the flexibility of $\Psi$-RAG, the abstraction and generation latency can also be mitigated by employing more lightweight LLMs. As shown in \cref{tab:efficiency}, L3-8B and G3-27B greatly accelerate $\Psi$-RAG with minimal performance loss.

Considering the limitations of similarity ranking when applied to industrial-scale corpora with 10M+ and even 100M+ tokens, we further propose two efficient tree-building extensions for $\Psi$-RAG: {\bf Bucketing} and {\bf Approximate Nearest Neighbor Search with Hierarchical Navigable Small World (HNSW)}~\cite{hnsw}. See Appendix \ref{extensions} for details. 

{\bf More experiment results.}
Appendix \ref{experimentapp} provides more extensive results, including case studies, failure studies, and more analyses towards the effectiveness of $\Psi$-RAG components. 

\section{Conclusion}

We propose $\Psi$-RAG, a hierarchical Tree-RAG framework supporting corpus-level indexing and cross-document multi-hop user questions. 
By a hierarchical clustering-based abstract tree index and a multi-granular agentic retriever, 
$\Psi$-RAG addresses key limitations of $k$-means-type tree indexes such as poor distribution adaptability, structural isolation, and coarse abstraction. 
Extensive experimental results demonstrate that $\Psi$-RAG achieves substantial gains across tasks of diverse granularities, promoting the development of tree-RAG in the rapidly evolving era of LLMs.

\section*{Acknowledgements}

This work was partially supported by the Guangdong Provincial Natural Science Foundation General Program (Grant No. 2026A1515012118).

We thank Linxiao Cao for his assistance in the comparative experiment. We also thank the anonymous reviewers, the Associate Chair, the Program Chair, etc. for their constructive feedback and recognition of this work.

\section*{Impact Statement}

This work focuses on retrieval-augmented generation (RAG) for large language models (LLMs) with a structured index. While this work may have broad applications in knowledge‑intensive natural language processing tasks, we do not identify any societal impacts that extend beyond those common to LLMs and retrieval systems.


\bibliography{ref}
\bibliographystyle{icml2026}

%%%%%%%%%%%%%%%%%%%%%%%%%%%%%%%%%%%%%%%%%%%%%%%%%%%%%%%%%%%%%%%%%%%%%%%%%%%%%%%
%%%%%%%%%%%%%%%%%%%%%%%%%%%%%%%%%%%%%%%%%%%%%%%%%%%%%%%%%%%%%%%%%%%%%%%%%%%%%%%
% APPENDIX
%%%%%%%%%%%%%%%%%%%%%%%%%%%%%%%%%%%%%%%%%%%%%%%%%%%%%%%%%%%%%%%%%%%%%%%%%%%%%%%
%%%%%%%%%%%%%%%%%%%%%%%%%%%%%%%%%%%%%%%%%%%%%%%%%%%%%%%%%%%%%%%%%%%%%%%%%%%%%%%
\newpage
\appendix
\onecolumn
\renewcommand{\contentsname}{Appendix Contents}
\tableofcontents
\let\addcontentsline\OriginalAddContentsLine
\newpage
\twocolumn
\section{Algorithms of $\Psi$-RAG} \label{alg}
\cref{alg:index} shows the algorithm of Abstract Tree Indexing, while \cref{alg:retrieve} and \cref{alg:agentretrieve} show the algorithm of Multi-granular Agentic Retrieval. 

\begin{algorithm}
  \caption{Abstract Tree Indexing}
  \label{alg:index}
  \begin{algorithmic}
    \STATE {\bfseries Input:} pairwise similarity matrix $\mathbf{S}$
    \STATE Descending sort $s(i,j)$ for all $1 \le i < j \le n$
    \STATE Initialize $\mathcal{T}=(\mathcal{V}=\emptyset, \mathcal{E}=\emptyset)$
    \FOR{$rank(u,v)=1$ {\bfseries to} $n-1$, $(u,v) \in \mathbf{S}$}
    \IF{$p(u)=p(v)=\emptyset$}
    \STATE Add node $a$ in $\mathcal{V}$ and set $c(a) = \{u, v\}$
    \ELSIF{$root(u) \ne root(v)$ \AND $\delta(u)=\delta(v)$ }
    \STATE Add node $a$ in $\mathcal{V}$ and set $c(a) = \{root(u), root(v)\}$
    \ELSIF{$root(u) \ne root(v)$ \AND $\delta(u)>\delta(v)$ }
    \STATE $c(p^{\delta(v)+1}(u)) \leftarrow c(p^{\delta(v)+1}(u)) \cup \{root(v)\}$
    \ELSIF{$root(u) \ne root(v)$ \AND $\delta(u)<\delta(v)$ }
    \STATE $c(p^{\delta(u)+1}(v)) \leftarrow c(p^{\delta(u)+1}(v)) \cup \{root(u)\}$
    \ENDIF
    \ENDFOR
    \STATE Tree rebalancing by splitting overload nodes
    \STATE Generate abstract text for every abstract node
    \STATE {\bfseries Output:} abstract tree $\mathcal{T}$
  \end{algorithmic}
\end{algorithm}

\newpage

\begin{algorithm}
  \caption{Tree Retrieval $r(q, \mathcal{T})$}
  \label{alg:retrieve}
  \begin{algorithmic}
    \STATE {\bfseries Input:} query $q$, tree index $\mathcal{T}$, number of retrieved chunks $k$, maximum retrieval step $i_{\max}$
    \STATE Initialize candidate nodes $\mathcal{V}_0 = \{root(\mathcal{T})\}$
    \FOR{$l=0$ {\bfseries to} $\delta(\mathcal{T})$}
        \IF{$l=\delta(\mathcal{T})$ \AND $\vert \{u \vert \delta(u)=l\} \vert \le k$}
        \STATE $\mathcal{D}^*$ = $\ell(\mathcal{T})$
        \STATE {\bfseries break}
        \ELSIF{$\vert \{u \vert \delta(u)=l\} \vert \le k$}
        \STATE $\mathcal{V}_{l+1} = \{u \vert \delta(u)=l+1\}$
        \STATE {\bfseries continue}
        \ENDIF
        \STATE \COMMENT{Top-$k$ nearest neighbor search}
        \STATE $\mathcal{U}_l = topk(\mathcal{V}_l, q; k)$ 
        \IF{$l=\delta(\mathcal{T})$}
        \STATE $\mathcal{D}^*$ = $\mathcal{U}_l$
        \ELSE
        \STATE $\mathcal{V}_{l+1} = \bigcup_{u \in \mathcal{U}_l} c(u)$
        \ENDIF
    \ENDFOR
    \STATE {\bfseries Output:} top-$k$ chunks $\mathcal{D}^*$
  \end{algorithmic}
\end{algorithm}

\begin{algorithm}
  \caption{Multi-granular Agentic Retrieval}
  \label{alg:agentretrieve}
  \begin{algorithmic}
    \STATE {\bfseries Input:} user query $q$, tree index $\mathcal{T}$, sparse index $\mathcal{I}$, LLM agent $f(\cdot;I(\cdot))$, number of retrieved chunks $k$
    \STATE \COMMENT{Hybrid retrieval}
    \STATE Initialize $\mathcal{D}^*_0 = r(q,\mathcal{T}) \cup r(q,\mathcal{I})$ 
    \STATE Current retrieval step $i=1$
    \REPEAT
        \IF{$i=i_{\max}$}
            \STATE Set $I(\mathcal{D}^*_i)$ such that $\texttt{<action>}=\texttt{<answer>}$
        \ENDIF
        \STATE Generate $a_{i} = (R_{i}, \texttt{<action>}, \cdot)$ from \eqref{eq:iterretrieve}
        \IF{$a_{i} = (R_{i}, \texttt{<retrieve>}, q'_{i})$}
            \STATE $\mathcal{D}^*_{i} = r(q'_{i},\mathcal{T}) \cup r(q'_{i},\mathcal{I})$ 
            \STATE $i\leftarrow i+1$
        \ENDIF
    \UNTIL{$a_{i} = (R, \texttt{<answer>}, y)$}
    \STATE {\bfseries Output:} answer $y$
  \end{algorithmic}
\end{algorithm}

\onecolumn

\section{Theoretical Proofs}

\subsection{Proof of \cref{thm:1}} \label{proof1}

\noindent\textbf{Theorem \ref{thm:1}} (Uniform effect of $k$-means \cite{kmeansuniform12}). 
{\it 
Given $\mathcal{L}$ in \eqref{kmeansloss}, we have
}
\begin{equation}
\mathcal{L} = - A\sum_{1 \le i < j \le k}{n_in_j {\Vert \mu_i - \mu_j \Vert}^2} + B
\end{equation}
{\it 
where $A$ and $B$ are constants only relative to the data. 
}
\begin{proof}
Let $d(c_i, c_j) = \sum_{x_i \in c_i, x_j \in c_j}{{\Vert x_i - x_j \Vert}^2}$. We first denote the sum of all pairwise distances of data samples $x \in \ell(\mathcal{T})$ within $k$ clusters as
\begin{equation}
\mathcal{P}_k=\sum_{i=1}^{n}{\sum_{j=1}^{n}{{\Vert x_i - x_j \Vert}^2}}=\sum_{i=1}^{k}{d(c_i, c_i)} + 2 \sum_{1\le i < j \le k}{d(c_i,c_j)} \label{pkdef}
\end{equation}
Then, we are going to prove
\begin{equation}
\mathcal{P}_k = \sum_{i=1}^{k}{\frac{n}{n_i}d(c_i, c_i)} + 2\mathcal{L}_k^{\leftrightarrow}, \label{pklemma}
\end{equation}
where
\begin{equation}
\mathcal{L}_k^{\leftrightarrow} = \sum_{1 \le i < j \le k}{n_in_j {\Vert \mu_i - \mu_j \Vert}^2}
\end{equation}
with mathematical induction. 

For $k=1$, we have $\mathcal{L}_1^{\leftrightarrow}=0$ as condition $1 \le i < j \le k$ cannot be met. So $\mathcal{P}_1 = \sum_{x,y\in \ell(\mathcal{T})}{{\Vert x - y \Vert}^2}$ which meets \eqref{pklemma}. 

For $k=2$, we have
\begin{align}
\mathcal{P}_2 &= d(c_1,c_1) + d(c_2,c_2) + 2d(c_1,c_2) \\
&= \sum_{1 \le i < j \le n_1}{{\Vert x_i - x_j \Vert}^2} + \sum_{1 \le i < j \le n_2}{{\Vert x_i - x_j \Vert}^2} + 2 \sum_{1\le i \le n_1}{\sum_{1\le j \le n_2}{{\Vert x_i - x_j \Vert}^2}} \\
&= 2(n_1-1)\sum_{i=1}^{n_1}{{\Vert x_i \Vert}^2} - 4 \sum_{1 \le i < j \le n_1}{x_i x_j} + 2(n_1-2)\sum_{i=1}^{n_2}{{\Vert x_i \Vert}^2} - 4 \sum_{1 \le i < j \le n_2}{x_i x_j} \notag \\ 
&\quad + 2n_2\sum_{i=1}^{n_1}{{\Vert x_i \Vert}^2} + 2n_1\sum_{i=1}^{n_2}{{\Vert x_i \Vert}^2} -4 \sum_{1\le i \le n_1}{\sum_{1\le j \le n_2}{x_i x_j}} \\
&= \left(\frac{2n(n_1-1)}{n_1}\sum_{i=1}^{n_1}{{\Vert x_i \Vert}^2} - 4 \cdot \frac{n}{n_1} \sum_{1 \le i < j \le n_1}{x_i x_j} \right) + \left(\frac{2n(n_2-1)}{n_2}\sum_{i=1}^{n_2}{{\Vert x_i \Vert}^2} - 4 \cdot \frac{n}{n_2} \sum_{1 \le i < j \le n_2}{x_i x_j}\right) \notag \\
&\quad + \left( \frac{2n_2}{n_1}\left(\sum_{i=1}^{n_1}{{\Vert x_i \Vert}^2} + 2 \sum_{1 \le i < j \le n_1}{x_i x_j} \right) + \frac{2n_1}{n_2}\left(\sum_{i=1}^{n_2}{{\Vert x_i \Vert}^2} + 2 \sum_{1 \le i < j \le n_2}{x_i x_j} \right) -4 \sum_{1\le i \le n_1}{\sum_{1\le j \le n_2}{x_i x_j}} \right) \\
&= \frac{n}{n_1}\sum_{1 \le i < j \le n_1}{{\Vert x_i - x_j \Vert}^2} + \frac{n}{n_2}\sum_{1 \le i < j \le n_2}{{\Vert x_i - x_j \Vert}^2} + 2n_1n_2 {\Vert \mu_1 - \mu_2 \Vert}^2 \\
&= \frac{n}{n_1}d(c_1,c_1) + \frac{n}{n_2}d(c_2,c_2) + 2\mathcal{L}_2^{\leftrightarrow}
\end{align}
which meets \eqref{pklemma}. 

Assume \eqref{pklemma} holds with $k=K-1$. Then for $k=K$, we have
\begin{equation}
\mathcal{P}_K = \mathcal{P}_{K-1}^{\hat{i}} + d(c_i, c_i) + 2\sum_{1 \le j \ne i \le K}{d(c_i, c_j)} \label{pkiter}
\end{equation}
where $\mathcal{P}_{K-1}^{\hat{i}}$ is the sum of pairwise distances within $K-1$ clusters excluding cluster $c_i$ for $i =1,2,\cdots,K$. 

We sum up \eqref{pkiter} for every cluster $c_i$ and get
\begin{equation}
K\mathcal{P}_K = \sum_{i=1}^{K}{\mathcal{P}_{K-1}^{\hat{i}}}+\sum_{i=1}^{K}{d(c_i,c_i)} + 4 \sum_{1 \le i < j \le K}{d(c_i, c_j)}
\end{equation}
With our assumption:
\begin{equation}
\mathcal{P}_{K-1}^{\hat{i}} = \sum_{1 \le j \ne i \le K}{\frac{n-n_i}{n_j}d(c_j,c_j)} + 2 \sum_{1 \le j < \kappa \ne i \le K}{n_jn_\kappa {\Vert \mu_j-\mu_\kappa \Vert}^2}
\end{equation}
So
\begin{equation}
\sum_{i=1}^{K}{\mathcal{P}_{K-1}^{\hat{i}}} = (K-2)\left(\sum_{i=1}^{K}\frac{n}{n_i}d(c_i, c_i) + 2\sum_{1 \le i < j \le K}{n_in_j} {\Vert \mu_i-\mu_j \Vert}^2 \right) + \sum_{i=1}^{k}{d(c_i, c_i)},
\end{equation}
and from \eqref{pkdef}, we have
\begin{align}
K\mathcal{P}_K &= (K-2)\left(\sum_{i=1}^{K}\frac{n}{n_i}d(c_i, c_i) + 2\mathcal{L}_K^{\leftrightarrow} \right) + 2\sum_{i=1}^{k}{d(c_i, c_i)} + 4\sum_{1 \le i < j \le K}{d(c_i, c_j)} \\
&= (K-2)\left(\sum_{i=1}^{K}\frac{n}{n_i}d(c_i, c_i) + 2\mathcal{L}_K^{\leftrightarrow} \right) + 2\mathcal{P}_K \\
\Rightarrow \quad \mathcal{P}_K &= \sum_{i=1}^{K}\frac{n}{n_i}d(c_i, c_i) + 2\mathcal{L}_K^{\leftrightarrow}
\end{align}
So \eqref{pklemma} holds. With \eqref{pklemma}, we can now write $\mathcal{L}$ as 
\begin{align}
\mathcal{L} &= \sum_{i=1}^{k} {\sum_{x_i \in c_i}{\Vert x_i - \mu_i \Vert}^2} = \frac{1}{2}\sum_{i=1}^{k}{\frac{d(c_i,c_i)}{n_i}}\\
&= \frac{1}{2n}\sum_{i=1}^{k}{\left(\frac{n}{n_i}d(c_i,c_i)+2\mathcal{L}_{k}^{\leftrightarrow}\right)} - \frac{1}{n}\sum_{i=1}^{k}\mathcal{L}_{k}^{\leftrightarrow} \\
&= - \frac{1}{n}\mathcal{L}_{k}^{\leftrightarrow} + \frac{\mathcal{P}_k}{2n},
\end{align}
which completes the proof with $A=-\frac{1}{n}$ and $B=\frac{\mathcal{P}_k}{2n}$. 
\end{proof}
{\it Remark.} \cref{thm:1} proves that the objective of $k$-means is essentially maximizing $\mathcal{L}_{k}^{\leftrightarrow}=\sum_{1 \le i < j \le k}{n_in_j {\Vert \mu_i - \mu_j \Vert}^2}$. If we assume $n_in_j$ and ${\Vert \mu_i - \mu_j \Vert}^2$ is independent, the maximization of $n_in_j$ leads to $n_1=n_2=\cdots=n_k=n/k$.

\subsection{Proof of \cref{thm:2}} \label{proof2}

We first introduce necessary lemmas for \cref{thm:2} and \cref{thm:3}.

\begin{lemma}[Modularity of Dasgupta's cost \cite{dasgupta}] \label{lem:1}
Let $\mathcal{T}_a$ be the subtree in $\mathcal{T}$ rooted at node $a$. If a different tree $\mathcal{T}'$ is constructed by replacing $\mathcal{T}_a$ with another subtree $\mathcal{T}'_a$ satisfying $\ell(\mathcal{T}'_a) = \ell(\mathcal{T}_a)$, then
\begin{equation}
c(\mathcal{T}') = c(\mathcal{T}) - c(\mathcal{T}_a) + c(\mathcal{T}'_a)
\end{equation}
\end{lemma}
\begin{proof}
For any leaf pair where one leaf is inside $\mathcal{T}'_a$ and the other is outside, the leaf count of its LCA is given by
\begin{equation}
c_{\text{cross}}(\mathcal{T}'_a, \mathcal{T}' \backslash \mathcal{T}'_a) = \sum_{u \in \mathcal{T}'_a, v \in \mathcal{T}' \backslash \mathcal{T}'_a}{\vert \ell(\text{LCA}(u,v)) \vert} = n'_a + n_1 + n_2 + \cdots
\end{equation}
Here, the total number of leaves in the current subtree $n'_a = n_a$ remains unchanged regardless of the subtree structure. Therefore: 
\begin{align} 
c(\mathcal{T}') &= c(\mathcal{T}'_a) + c(\mathcal{T}' \backslash \mathcal{T}'_a )+ c_{\text{cross}}(\mathcal{T}'_a, \mathcal{T}' \backslash \mathcal{T}'_a)\\
&= c(\mathcal{T}'_a) + c(\mathcal{T} \backslash \mathcal{T}_a ) + c_{\text{cross}}(\mathcal{T}_a, \mathcal{T} \backslash \mathcal{T}_a)\\
&= c(\mathcal{T}'_a) + c(\mathcal{T}) - c(\mathcal{T}_a)
\end{align}
where $c(\mathcal{T} \backslash \mathcal{T}_a)$ denotes the sum of costs among the subtrees of $\mathcal{T}$ excluding $\mathcal{T}_a$.
\end{proof}

{\it Remark.} The modularity of Dasgupta's cost allows us to, without loss of generality, decompose the problem of computing the global cost into analyzing cost differences within corresponding subtrees.

\begin{figure*}[ht]
  \begin{center}
    \centerline{\includegraphics[width=0.9\textwidth]{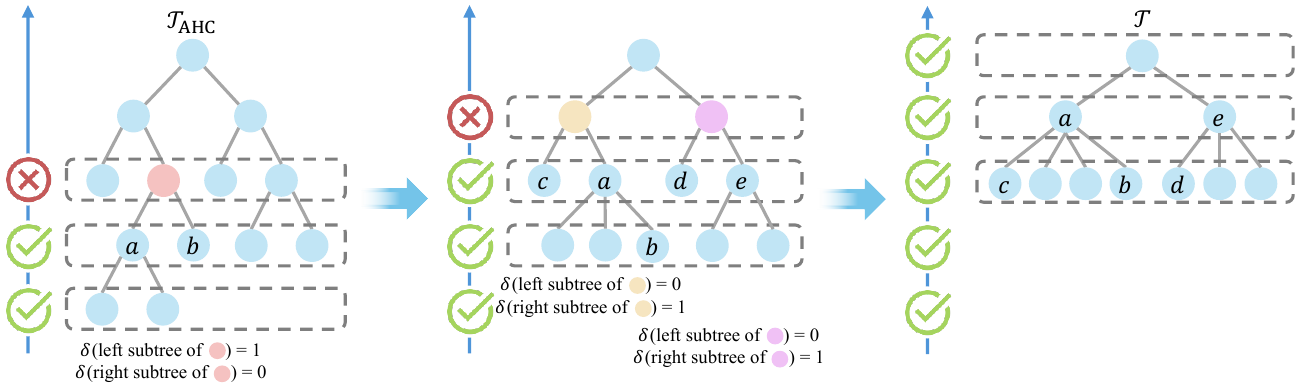}}
    \caption{
Illustration of $\psi(\cdot)$: bottom-up traversal and collapsing. Note that both $\mathcal{T}_{\text{AHC}}$ and $\mathcal{T}$ here are unordered, so the positions of $b$ and $c$ in $\mathcal{T}$ do not matter as long as they are children of $a$. 
    }
    \label{fig:hc2psi}
  \end{center}
  \vskip -0.2in
\end{figure*}

Lemma \ref{lem:3} establishes the connection between the equi-depth multifurcating tree of $\Psi$-RAG and the binary tree of AHC.

\begin{lemma}[Connection between AHC and $\Psi$-RAG] \label{lem:3}
There exists a surjection $\psi: \mathcal{T}_\text{AHC} \mapsto \mathcal{T}$ that maps every binary tree $\mathcal{T}_\text{AHC} $ to an equi-depth multifurcating tree $\mathcal{T}$ by iteratively performing node collapsing. 
\end{lemma}
\begin{proof} 
The proof is intuitive. First, the merging process of $\Psi$-RAG is identical to that in an AHC tree: both of them group two nodes under a common parent. 

We now focus on collapsing. As illustrated in \cref{fig:hc2psi}, the process $\psi(\mathcal{T}_{\text{AHC}})$ can be viewed as a bottom-up layer-by-layer traversal process. Whenever a node is found whose subtrees have unequal depths, the root of the shallower subtree is collapsed to the corresponding node in the deepest subtree. This process iterates until every node has subtrees of equal depth. Therefore, any $\mathcal{T}_{\text{AHC}}$ can always be mapped to an equi-depth multifurcating tree $\mathcal{T} = \psi(\mathcal{T}_{\text{AHC}})$. 
% This step is also equivalent to the node collapsing process in $\Psi$-RAG. 
\end{proof}
{\it Remark.} Lemma \ref{lem:3} allows us to decompose $\psi(\mathcal{T}_{\text{AHC}})$ bottom-up into a sequence of collapsing operations. Combined with Lemma \ref{lem:1}, it implies that $\Delta c = c(\psi(\mathcal{T}_{\text{AHC}})) - c(\mathcal{T}_{\text{AHC}})$ can be decomposed into a layer-by-layer summation of the cost increments introduced by node collapse. 
Moreover, it is obvious that Dasgupta's cost is invariant to the order of children of a certain node, so $\Delta c$ is invariant to the order of collapsing within a certain layer as well. 

In the following analysis, we follow the original paper \cite{dasgupta} by considering a computable scenario where the distance $d = 1$ for all pairs of chunks. 
The subsequent conclusions demonstrate that this assumption does not affect our conclusion of $\Psi$-RAG's property. 
To analyze the non-uniformity preference of $\Psi$-RAG, Lemma \ref{lem:4} quantifies the increment in Dasgupta's cost introduced by node collapse. 

\begin{figure*}[ht]
  \begin{center}
    \centerline{\includegraphics[width=\textwidth]{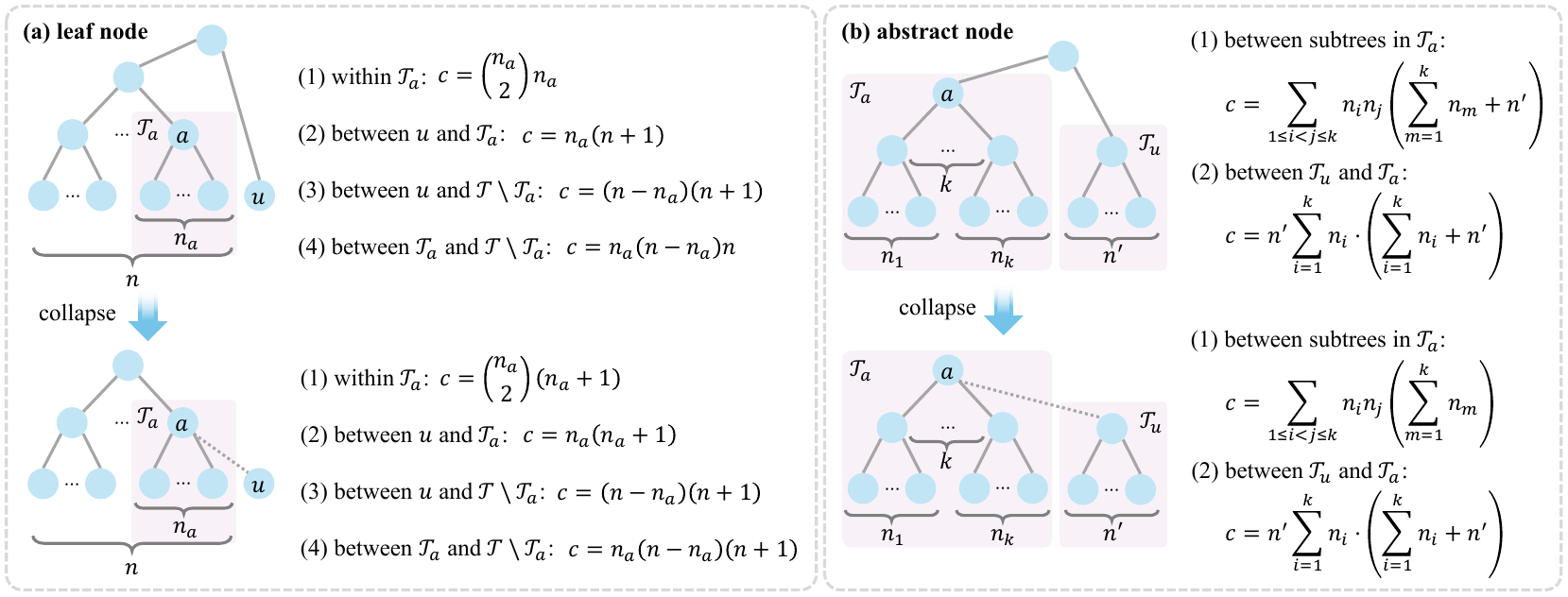}}
    \caption{
Decomposition of Dasgupta's cost for (a) leaf node collapse and (b) abstract node collapse. The Dasgupta's costs of unmentioned components in the figure obviously remain invariant before and after the collapse.
    }
    \label{fig:collapsecost}
  \end{center}
  \vskip -0.2in
\end{figure*}

\begin{lemma}[Increment of Dasgupta's cost of node collapse] \label{lem:4}
The following holds:

(1) For a leaf node collapse within a tree $\mathcal{T}$, let the total number of leaves be $\vert \ell(\mathcal{T}) \vert = n$, and let $\mathcal{T}_a$ be the subtree rooted at the abstract node $a$ to which the leaf node $u$ is being linked, with $\vert \ell(\mathcal{T}_a) \vert = n_a$. This introduces an increment of Dasgupta's cost $\Delta c =\binom{n_a}{2}$.

(2) For an abstract node collapse within a tree $\mathcal{T}$, let the number of leaves of the subtree to be collapsed $\mathcal{T}_u$ be $n'$. Let $\mathcal{T}_a$ be the subtree rooted at the abstract node $a$ to which $\mathcal{T}_u$ is being linked. $\mathcal{T}_a$ has $k$ branches with the number of leaves $n_1, n_2, \cdots, n_k$, respectively. This introduces an increment of Dasgupta's cost $\Delta c = n' \sum_{1 \le i < j \le k} n_i n_j$.
\end{lemma}
\begin{proof}
(1) From Fig~\ref{fig:collapsecost} (a), 
\begin{equation}
\Delta c = \underbrace{\binom{n_a}{2}(n_a+1-n_a)}_{\Delta c(\text{within } \mathcal{T}_a)} + \underbrace{n_a(n_a + 1) - n_a(n + 1)}_{\Delta c(\text{between }u\text{ and }\mathcal{T}_a)}  + \underbrace{n_a(n-n_a)(n+1-n)}_{\Delta c(\text{between }\mathcal{T}_a\text{ and }\mathcal{T} \backslash \mathcal{T}_a)} = \binom{n_a}{2}
\end{equation}
(2) From Fig~\ref{fig:collapsecost} (b), 
\begin{equation}
\Delta c = \underbrace{\sum_{1 \le i < j \le k}{n_i n_j \left( \sum_{m=1}^{k}n_m + n' - \sum_{m=1}^{k}n_m \right)}}_{\Delta c(\text{between subtrees in }\mathcal{T}_a)} = n'\sum_{1 \le i < j \le k}{n_i n_j}
\end{equation}
which ends the proof.
\end{proof}
\begin{corollary} \label{crl:5}
By the Hockey-stick identity \cite{hockeystick}, the total cost increment from performing $n-2$ consecutive leaf collapses on a binary branch (resulting in $n \ge 3$ leaves) is
\begin{equation}
\Delta c = \sum_{i=2}^{n-1}{\binom{i}{2}} = \binom{n}{3}
\end{equation}
The total cost increment from collapsing $K-k$ ($K > k \ge 2$) identical subtrees onto node $a$ (resulting in $K$ subtrees) is
\begin{equation}
\Delta c = \sum_{\kappa=k}^{K-1}{\left(n'\sum_{1 \le i < j \le \kappa}{n_i n_j}\right)} = n'\sum_{\kappa=k}^{K-1}{\sum_{1 \le i < j \le \kappa}{n_i n_j}}    
\end{equation}
\end{corollary}

Now we are ready for the following theorem. 

\begin{figure*}[ht]
  \begin{center}
    \centerline{\includegraphics[width=\textwidth]{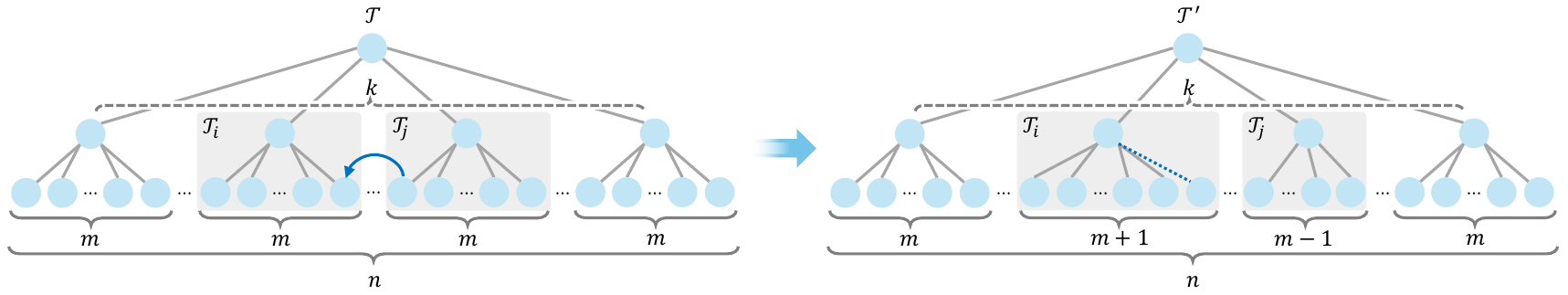}}
    \caption{
  An example of a uniform cluster tree (left) and moving a node from $\mathcal{T}_j$ to another subtree $\mathcal{T}_i$ (right).
    }
    \label{fig:nonuniform}
  \end{center}
  \vskip -0.2in
\end{figure*}

\noindent\textbf{Theorem \ref{thm:2}} (Non-uniformity preference of $\Psi$-RAG).
{\it 
Consider a 3-layer equi-depth multifurcating tree $\mathcal{T}$ with a perfectly uniform leaf distribution among its subtrees $\mathcal{T}_1, \cdots, \mathcal{T}_k$: $n_1=n_2=\cdots=n_k=n/k$, where $k \ge 3$, $n \ge 3k$. A new tree $\mathcal{T}'$ is constructed by moving one leaf from $\mathcal{T}_j$ to another subtree $\mathcal{T}_i$. Then: \\ %  such that $n'_i=n/k+1, n'_j=n/k-1$
(1) $c(\mathcal{T}') < c(\mathcal{T})$; \\
(2) Let $\mathcal{T} = \psi(\mathcal{T}_{\text{AHC}})$ and $\mathcal{T'}=\psi(\mathcal{T}'_{\text{AHC}})$. Denote by $\Delta c$ and $\Delta c'$ the increments of Dasgupta's cost on $\mathcal{T}$ and $\mathcal{T}'$ introduced by $\psi(\cdot)$, respectively. Then we have $c(\mathcal{T}') - c(\mathcal{T}) = \Delta c' - \Delta c$.
}

\begin{proof}
(1) We visualize the scenario in \cref{fig:nonuniform}. Let $m=n/k$ and we can know from Lemma \ref{lem:1} that
\begin{equation}
c(\mathcal{T}) = k\binom{m}{2}m + \binom{k}{2}m^2n,
\end{equation}
\begin{align}
c(\mathcal{T}') &= (k-2)\binom{m}{2}m + \binom{m+1}{2}\left(m+1\right) + \binom{m-1}{2}\left(m-1\right) \notag \\
&\quad +\underbrace{\binom{k-2}{2}m^2n}_{0 \text{ if } k=3} + \left(m+1\right)\left(m-1\right)n+(k-2)(m+1)mn+(k-2)(m-1)mn
\end{align}
So
\begin{equation}
c(\mathcal{T}')-c(\mathcal{T})=3m-n-1,
\end{equation}
which yields $c(\mathcal{T}') < c(\mathcal{T})$ if $k \ge 3$.

(2) According to Lemma \ref{lem:3}, we can decompose $\psi(\cdot)$ bottom-up into multiple leaf node collapses and multiple abstract node collapses. From Corollary \ref{crl:5}, the total cost increments introduced by $\psi(\mathcal{T}_{\text{AHC}})$ and $\psi(\mathcal{T}'_{\text{AHC}})$ are
\begin{equation}
\Delta c = k\binom{m}{3} + m\sum_{i=2}^{k-1}{\binom{i}{2}m^2} = k\binom{m}{3} + \binom{k}{3}m^3,
\end{equation}
\begin{align}
\Delta c' &= (k-2)\binom{m}{3} + \underbrace{\binom{m-1}{3}}_{0 \text{ if } m=3} + \binom{m+1}{3} + \underbrace{\binom{k-2}{3}m^3}_{0 \text{ if } k\le4} + \underbrace{\binom{k-2}{2}(m-1)m^2}_{0 \text{ if } k=3} \notag \\
& \quad+ (m+1)\left(\underbrace{\binom{k-2}{2}m^2}_{0 \text{ if } k=3} + (k-2)m(m-1) \right)
\end{align}
So we have
\begin{equation}
\Delta c' - \Delta c = 3m-n-1 = c(\mathcal{T}') - c(\mathcal{T})
\end{equation}
which ends the proof.
\end{proof}

{\it Remark.} \cref{thm:2} proves that while traditional AHC assigns the same Dasgupta's cost to the two different tree structures ($c(\mathcal{T}) - \Delta c = c(\mathcal{T}') - \Delta c'$), the collapsing process in $\Psi$-RAG results in a lower Dasgupta's cost for the tree with a non-uniform leaf distribution. This indicates that there is no uniform effect for $\Psi$-RAG even when pairwise distances between leaf nodes are unsorted. 
Moreover, as the computation of Dasgupta's cost can be decomposed into that of each subtree, our conclusions can be generalized to trees with $\delta > 3$.

\subsection{Proof of \cref{thm:3}}  \label{proof3}

\noindent\textbf{Theorem \ref{thm:3}} (Distribution adaptability of $\Psi$-RAG).
{\it
Consider a 3-layer equi-depth multifurcating tree $\mathcal{T}$ with subtrees $\mathcal{T}_1, \cdots, \mathcal{T}_k$ of any size where $k \ge 3$. A subtree of a minor leaf cluster $\mathcal{T}_i$ has $2 \le n_i < n_j, \forall 1 \le j \ne i \le k$. A new tree $\mathcal{T}'$ is constructed by moving one leaf from some subtree $\mathcal{T}_j$ to the minor subtree $\mathcal{T}_i$. Then: \\
(1) $c(\mathcal{T}') \ge c(\mathcal{T})$, and $c(\mathcal{T}') = c(\mathcal{T})$ iff $n_j=n_i + 1$; \\
(2) Let $\mathcal{T} = \psi(\mathcal{T}_{\text{AHC}})$,  $\mathcal{T'}=\psi(\mathcal{T}'_{\text{AHC}})$ and $\Delta c, \Delta c'$ defined as in \cref{thm:2}. Then we have $\Delta c' \ge \Delta c$, and $\Delta c' = \Delta c$ iff $n_j=n_i + 1$.
}

\begin{figure*}[ht]
  \begin{center}
    \centerline{\includegraphics[width=0.95\textwidth]{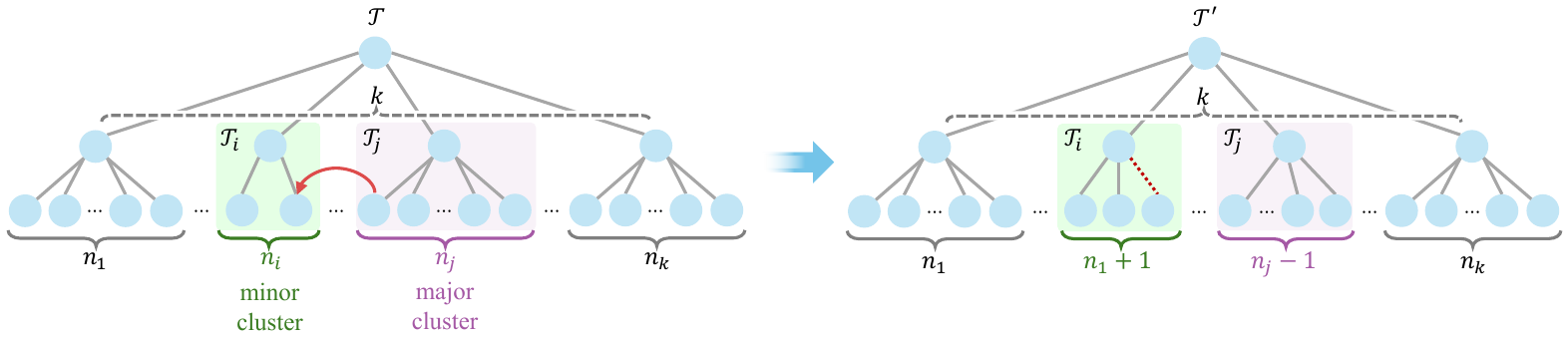}}
    \caption{
  An example of a tree with a minor subtree $\mathcal{T}_i$ (left) and moving a node from a major subtree $\mathcal{T}_j$ to $\mathcal{T}_i$ (right).
    }
    \label{fig:skewed}
  \end{center}
  \vskip -0.2in
\end{figure*}

\begin{proof}
(1) We visualize the scenario in \cref{fig:skewed}. So we have
\begin{equation}
c(\mathcal{T}) = \binom{n_i}{2}n_i+\binom{n_j}{2}n_j + n_i(n-n_i)n + n_j(n-n_j)n + C,
\end{equation}
\begin{equation}
c(\mathcal{T}') = \binom{n_i + 1}{2}(n_i+1)+\binom{n_j-1}{2}(n_j-1) + (n_i-1)(n-n_i+1)n + (n_j-1)(n-n_j+1)n + C
\end{equation}
where $C$ is a constant that remains invariant to the structure change. So
\begin{equation}
c(\mathcal{T}') - c(\mathcal{T}) = 2n(n_j-n_i-1)+\frac{3n_i^2+n_i-3n_j^2 +5n_j-2}{2}
\end{equation}
When $n_j=n_i+1$, $c(\mathcal{T}') - c(\mathcal{T}) = 0$. When  $n_j=n_i+x$ where $x>1$, then
\begin{align}
c(\mathcal{T}') - c(\mathcal{T}) &= 2n(x-1) + \frac{3n_i^2+n_i-3(n_i+x)^2+5(n_i+x)-2}{2} \\
&= \left(2n-3n_i-\frac{3x-2}{2}\right)(x-1)  \\
&= \left(2\left(n_i + n_i + x + \sum_{o\ne i, j}{n_o}\right)-3n_i-\frac{3x-2}{2}\right)(x-1) \\
&= \left(n_i +2 \sum_{o\ne i, j}{n_o} + \frac{x}{2} +1\right)(x-1) \\
&>0
\end{align}
which yields $c(\mathcal{T}') > c(\mathcal{T})$.

(2) From Corollary \ref{crl:5}, the total cost increments introduced by $\psi(\mathcal{T}_{\text{AHC}})$ and $\psi(\mathcal{T}'_{\text{AHC}})$ are
\begin{equation}
\Delta c = \binom{n_i}{3} +\binom{n_j}{3} + (n-n_i-n_j)n_in_j + C,
\end{equation}
\begin{equation}
\Delta c' = \binom{n_i+1}{3} + \underbrace{\binom{n_j-1}{3}}_{0 \text{ if } n_j=3} + (n-n_i-n_j)(n_i+1)(n_j-1) + C
\end{equation}
where $C$ is a constant that remains invariant to the structure change. So
\begin{equation}
\Delta c' - \Delta c = (n_j-n_i-1)\left(n-n_i-n_j - \frac{n_i+n_j}{2} +1\right)
\end{equation}
When $n_j=n_i+1$, $\Delta c' - \Delta c = 0$. When $n_j$ is larger, $n_j-n_i-1 > 0$, so we focus on the second part of the right hand side. We consider the minimum case of $n=n_i + n_j + n_j$ and
\begin{equation}
\Delta c' - \Delta c = (n_j-n_i-1)\left(\frac{n_j-n_i}{2} +1\right) >0
\end{equation}
which ends the proof.
\end{proof}

{\it Remark. } \cref{thm:3} proves that $\Psi$-RAG incurs a larger cost to assign a major cluster node to a minor cluster. For a minor class in a corpus with cluster $c = \{ u \vert \exists v \in c, d(u,v)<d(u,w) \ \forall w \notin c \}$, it always form a minor subtree in $\mathcal{T}$ following the similarity ranking and iterative merging \& collapse processes of $\Psi$-RAG. In conclusion, $\Psi$-RAG tends to preserve the size of minor classes in a skewed corpus. As the computation of Dasgupta's cost can be decomposed into that of each subtree, our conclusions can be generalized to trees with $\delta > 3$. 

\begin{table}[t]
\caption{Statistics of datasets. }
\vskip -0.1in
\label{tab:data}
\begin{center}
\begin{small}
\begin{tabular}{lccrrr}
\toprule[1.2pt]
Dataset & Task type & Index scale & \#Queries & \#Documents & Average \#tokens per document \\
\midrule
NQ & Single-hop QA & Corpus & 1,000 & 9,633 & 141.330 \\
PopQA & Single-hop QA & Corpus & 1,000 & 8,676 & 136.536 \\
HotpotQA & Multi-hop QA & Corpus & 1,000 & 9,811 & 131.602 \\
2Wiki & Multi-hop QA & Corpus & 1,000 & 6,119 & 110.436 \\
MuSiQue & Multi-hop QA & Corpus & 1,000 & 11,656 & 114.516 \\
MultiHop-RAG & Multi-hop QA & Corpus & 1,000 & 609 & 2346.064 \\
NarrativeQA & Narrative QA & Passage & 293 & 10 & 62,824.500 \\
$\infty$-LongBook & Narrative QA & Passage & 100 & 19 & 214,711.895 \\
QMSum & Summarization & Passage & 35 & 35 & 14,375.543 \\
WCEP & Summarization & Passage & 119 & 119 & 9,567.235 \\
\bottomrule[1.2pt]
\end{tabular}
\end{small}
\end{center}
\end{table}

\section{Detailed Experimental Setup} \label{setupapp}

\subsection{Dataset and Index Statistics} 

{\bf Statistics of datasets.} 
We provide the statistics of datasets used in our experiments in \cref{tab:data}. 

\begin{table}[t]
\caption{Statistics of abstract tree indexes of $\Psi$-RAG.}
\vskip -0.1in
\label{tab:treestats}
\begin{center}
\begin{small}
\begin{tabular}{lrrrrr}
\toprule[1.2pt]
Dataset & \#Trees & \makecell{\#Average\\leaf nodes} & \makecell{\#Average\\abstract nodes} & \makecell{Average\\tree depth} & \makecell{Average\\compression rate} \\
\midrule
NQ & 1 & 9,633 & 3,164 & 6 & 82.35\% \!\!{\small \faFile*[regular]}\!\! / 83.06\% \!\!{\small \faKey}\!\! \\
PopQA & 1 & 8,676 & 2,810 & 5 & 83.26\% \!\!{\small \faFile*[regular]}\!\! / 81.87\% \!\!{\small \faKey}\!\! \\
HotpotQA & 1 & 9,811 & 2,635 & 6 & 85.98\% \!\!{\small \faFile*[regular]}\!\! / 85.85\% \!\!{\small \faKey}\!\! \\
2Wiki & 1 & 6,119 & 1,704 & 5 & 82.14\% \!\!{\small \faFile*[regular]}\!\! / 82.89\% \!\!{\small \faKey}\!\! \\
MuSiQue & 1 & 11,656 & 2,938 & 6 & 84.79\% \!\!{\small \faFile*[regular]}\!\! / 83.09\% \!\!{\small \faKey}\!\! \\
MultiHop-RAG & 1 & 16,349 & 4,377 & 6 & 78.94\% \!\!{\small \faFile*[regular]}\!\! / 78.87\% \!\!{\small \faKey}\!\! \\
NarrativeQA & 10 & 411 & 85 & 3.5 & 90.71\% \!\!{\small \faFile*[regular]}\!\! / 98.25\% \!\!{\small \faKey}\!\! \\
$\infty$-LongBook & 19 & 2,446 & 488 & 4.2 & 84.36\% \!\!{\small \faFile*[regular]}\!\! / 84.31\% \!\!{\small \faKey}\!\! \\
QMSum & 35 & 155 & 41 & 3.3 & 82.19\% \!\!{\small \faFile*[regular]}\!\! / 82.33\% \!\!{\small \faKey}\!\! \\
WCEP & 119 & 109 & 33 & 3.2 & 76.32\% \!\!{\small \faFile*[regular]}\!\! / 76.35\% \!\!{\small \faKey}\!\! \\
\bottomrule[1.2pt]
\end{tabular}
\end{small}
\end{center}
\end{table}

{\bf Dataset split.} 
We follow RAPTOR's document splitting strategy, segmenting documents into chunks of approximate 100 tokens without cutting mid-sentence. Notably, we do not re-split datasets composed of similarly sized paragraphs: NQ, PopQA, HotpotQA, 2Wiki, and MuSiQue.

{\bf Statistics of abstract trees.} 
We provide the statistics of $\Psi$-RAG's abstract tree indexes for each dataset in \cref{tab:treestats}. The results show that for corpus-level datasets, the tree depth of $\Psi$-RAG stabilizes at 5 to 6 layers. This demonstrates $\Psi$-RAG's key advantage of generalization ability: it does not require prior assumptions about data, as it naturally uncovers the underlying semantic hierarchy.

% $\Psi$-RAG's abstracts can reach a compression rate of at least 76.35\%. 

\subsection{Configurations} 

{\bf Indexing setup.}
We set the maximum token length of the summative abstract for each abstract node to 100, and the maximum number of keywords of the keyword abstract to 20. 
We set the maximum number of children nodes to 40. Abstract nodes with excessive children will be split in half in the rebalancing process. 

The keyword abstract \!\!\faKey\!\! is used for $\Psi$-RAG in all non-comparative experiments in this paper, unless otherwise stated. 

{\bf Retrieval setup.} 
We mainly follow the multi-step setting of HippoRAG \cite{hipporag} for agentic retrieval. For QA tasks, top-10 retrieved passages from both sparse and dense tree retrievers are integrated and reranked to the final top-5 at each retrieval step. We maintain a maximum score list for retrieved document chunks at each step and take top-5 documents from the list for recall@$k$ evaluation. For multi-hop datasets, we limit the maximum number of retrieval attempts ($i_{\max}+1$) to be equal to the maximum hop count of user questions: 2 for HotpotQA and 2Wiki, 4 for MuSiQue and MultiHop-RAG. For single-hop datasets (NQ and PopQA), though, we allow one extra retrieval attempt for possible query reorganization. For summarization datasets, we perform single retrieval. The $k$ value is determined by the size of documents: top-20 for QMSum and top-15 for WCEP. 

{\bf LLM setup.} 
All experiments are conducted on up to 3 NVIDIA GeForce RTX 4090 GPUs with 48GB VRAM each. We use Ollama\footnote{\url{https://ollama.com/}} and Hugging Face as default LLM sources. We follow the majority of default settings of Ollama and Hugging Face models: for example, the context window is fixed to 4096 tokens. We set the decoding temperature to 0 and the penalty for generating repetitive content to 1.3 for every LLM. 
For the summarization task, the maximum LLM response token length is determined by the size of documents: 200 for QMSum and 150 for WCEP. 

\subsection{Instruction Prompts} 

We show LLM prompts used for generating summative abstract, keyword abstract, and R\&A agent below. For single-hop and multi-hop QA, we follow existing works and provide a one-shot in-context example.

\begin{small}
\begin{tcolorbox}[breakable, pad at break=2.5mm, before skip=1mm, after skip=1mm, width=1.0\linewidth, boxsep=2mm, arc=1.5mm, left=1.5mm, right=1.5mm, top=1.5mm, bottom=1mm, parskip=\baselineskip, title={LLM Prompt for \!\faFile*[regular]\!\! Summative Abstract}]

\begin{center}
======================== {\bf \{System Instruction\}} ========================
\end{center}

As an advanced document summarization assistant, your task is to summarize the given list of documents, including as many key details and relationships as possible. 

(1) Your answer should always start with ``Summary:''. NO UNNECESSARY WORDS: Always directly output the summary instead of useless modifiers like `Here is a summary of the given list ...'. 

(2) Be faithful to the original documents. Unnecessary paraphrasing and adding content not mentioned in the original documents are NOT allowed. 

(3) Your summary should be NO MORE THAN $\texttt{\{summary\_max\_length\}}$ WORDS. Keep in mind that a longer summary is not always better. 

An in-context example is provided below.

\begin{center}
======================== {\bf \{In-context example\}} ========================
\end{center} 

\faUser User input:

Hotel Tallcorn

The Hotel Tallcorn is located in Marshalltown, Iowa... 

The Oberoi Group

The Oberoi Group is a hotel company with its head office in Delhi...

Hotel Bond

Hotel Bond is a historic hotel, built in two stages in 1913 and 1921... 

Mohan Singh Oberoi

Rai Bahadur Mohan Singh Oberoi (15 August 1898 – 3 May 2002)...

Ritz-Carlton Jakarta

The Ritz-Carlton Jakarta is a hotel and skyscraper in Jakarta... \mbox{} \\

\faRobot Assistant response: 

The Hotel Tallcorn in Marshalltown, Iowa, built in 1928, is now an apartment building. It was part of one of the 
second-largest hotel sales in U.S. history in 1956. Hotel Bond is a historic hotel in Hartford, Connecticut, built in two stages. 
After facing competition, it went through multiple ownership changes and bankruptcies. It is now a Homewood Suites by Hilton, 
with its ballroom having reopened in 2001. The Oberoi Group, founded in 1934, is a luxury hotel company based in Delhi. 
It owns and operates over 30 hotels and two cruise ships. It was founded by Mohan Singh Oberoi, who started India's second-largest 
hotel company. The Ritz-Carlton Jakarta is a skyscraper hotel in Indonesia, opened in 2005. It consists of two towers, 
one for the hotel and one for apartments, and is located near the JW Marriott Hotel. \mbox{} \\

\end{tcolorbox}
\end{small}

\begin{small}
\begin{tcolorbox}[breakable, pad at break=2.5mm, before skip=1mm, after skip=1mm, width=1.0\linewidth, boxsep=2mm, arc=1.5mm, left=1.5mm, right=1.5mm, top=1.5mm, bottom=1mm, parskip=\baselineskip, title={LLM Prompt for \!\faKey\!\! Keyword Abstract}]

\begin{center}
======================== {\bf \{System Instruction\}} ========================
\end{center}

As an advanced document summarization assistant, your task is to summarize the given list of text paragraphs into multiple key phrases or words. These key phrases include: 

· Key phrases that overview the types of information covered in the text;

· Key phrases that reflect the common themes of these paragraphs;

· Key phrases that distinguish the specific content of each paragraph.

Requirements: 

(1) Generate key phrases mentioned above in order. All key phrases should be strictly formatted as ONE line of text, separated by commas, equivalent to the output format of Python code ``$\texttt{`,'.join(list\_of\_keywords)}$''. 

(2) Different key phrases should be distinct from one another.

(3) NO UNNECESSARY WORDS: Always output the direct key phrase list instead of useless modifiers like `Here are the key phrases according to the documents: ...' or `Information key phrases: ... Theme key phrases: ...'. 

(4) The number of key phrases should be NO MORE THAN $\texttt{\{max\_num\_keywords\}}$. 

An in-context example is provided below.

\begin{center}
======================== {\bf \{In-context example\}} ========================
\end{center} 

\faUser User input:

Hotel Tallcorn

The Hotel Tallcorn is located in Marshalltown, Iowa... 

The Oberoi Group

The Oberoi Group is a hotel company with its head office in Delhi...

Hotel Bond

Hotel Bond is a historic hotel, built in two stages in 1913 and 1921... 

Mohan Singh Oberoi

Rai Bahadur Mohan Singh Oberoi (15 August 1898 – 3 May 2002)...

Ritz-Carlton Jakarta

The Ritz-Carlton Jakarta is a hotel and skyscraper in Jakarta... \mbox{} \\

\faRobot Assistant response: 

historic hotels,hotel companies,bankruptcy and renovation,brand operations,Indian hospitality industry,Indonesia skyscraper,hotel locations,hotel building time,hotel architecture,hotel founders,hotel ownership history,Marshalltown Iowa,Hotel Tallcorn,Eppley Hotel Company,Sheraton sale,The Oberoi Group,Mohan Singh Oberoi,Hotel Bond,Hartford Connecticut,Bond Ballroom,Ritz-Carlton Jakarta \mbox{} \\

\end{tcolorbox}
\end{small}

\begin{small}
\begin{tcolorbox}[breakable, pad at break=2.5mm, before skip=1mm, after skip=1mm, width=1.0\linewidth, boxsep=2mm, arc=1.5mm, left=1.5mm, right=1.5mm, top=1.5mm, bottom=1mm, title={LLM Prompt for R\&A Agent in single-hop and multi-hop QA}]

\begin{center}
======================== {\bf \{System Instruction\}} ========================
\end{center}

As an advanced reading comprehension and information retrieval agent, your task is to answer the user question by retrieving documents and multi-step reasoning, following the instructions below: 

(1) Your response starts after ``Thought: '', where you will methodically break down the reasoning process, illustrating how you arrive at conclusions step-by-step. If you believe you have found the answer, conclude with ``Answer: $\texttt{<answer>}$'' to present a concise, definitive response.

(2) You can request for retrieving supporting documents from a relevant database a maximum of $\texttt{\{max\_retrieval\_time}\}$ times. Some retrieved documents are provided the first time. When the question involves content you do not understand, or current documents do not provide enough information, request retrieval by responding with ``Retrieve: $\texttt{<retrieve\_query>}$'', where $\texttt{<retrieve\_query>}$ is a sub-question derived from the user question, which you believe must be resolved to answer the user question.

(3) The retrieval results will be provided in the format ``Retrieved documents: ...'', and you should continue your reasoning based on the user question, your previous reasoning steps, and all retrieved documents. 
Obtain your answer with as few retrieval attempts as possible, which means you need to make your $\texttt{<retrieve\_query>}$ as precise as possible. If newly retrieved documents do not help, try reorganizing your $\texttt{<retrieve\_query>}$. 
If you run out of retrieval attempts and still cannot answer the question, answer with ``Not mentioned'' only. Do NOT provide additional uncertain information. 

(4) NO UNNECESSARY WORDS: After your reasoning, you MUST end your response with ``Answer: $\texttt{<answer>}$'' or ``Retrieve: $\texttt{<retrieve\_query>}$'' in a new line. $\texttt{<answer>}$ is in most cases a word or a phrase. Avoid ambiguous words like ``approximately''.

(5) Your ``Thought: '' response should be NO MORE THAN $\texttt{\{thought\_max\_length\}}$ WORDS. 

An in-context example is provided below.

\begin{center}
======================== {\bf \{In-context example\}} ========================
\end{center} 

\faUser User input:

Retrieved documents: 

Hotel Tallcorn

The Hotel Tallcorn is located in Marshalltown, Iowa... 

Hotel Bond

Hotel Bond is a historic hotel, built in two stages in 1913 and 1921... 

Ritz-Carlton Jakarta

The Ritz-Carlton Jakarta is a hotel and skyscraper in Jakarta...

User question: The Oberoi family is part of a hotel company that has a head office in what city?

Retrieval time remaining: $\texttt{3}$

Thought: \mbox{} \\

\faRobot Assistant response: 

According to the user question, I need to first figure out the hotel company that the Oberoi family is part of. However, the retrieved documents do not include any details regarding the Oberoi family. 

Retrieve: Which hotel company is the Oberoi family part of?  \mbox{} \\

\faUser User input:

Retrieved documents: 

The Oberoi Group

The Oberoi Group is a hotel company with its head office in Delhi...

Mohan Singh Oberoi

Rai Bahadur Mohan Singh Oberoi (15 August 1898 – 3 May 2002)...

User question: The Oberoi family is part of a hotel company that has a head office in what city?

Retrieval time remaining: $\texttt{2}$

Thought: \mbox{} \\

\faRobot Assistant response: 

According to ``Mohan Singh Oberoi'', Mohan Singh Oberoi is the founder and chairman of Oberoi Hotels \& Resorts. From ``The Oberoi Group'', Oberoi Hotels \& Resorts is a brand and its company, The Oberoi Group, is a hotel company with its head office in Delhi. 

Answer: Delhi

\end{tcolorbox}
\end{small}

\begin{small}
\begin{tcolorbox}[breakable, pad at break=2.5mm, before skip=1mm, after skip=1mm, width=1.0\linewidth, boxsep=2mm, arc=1.5mm, left=1.5mm, right=1.5mm, top=1.5mm, bottom=1mm, title={LLM Prompt for R\&A Agent in summarization}]

\begin{center}
======================== {\bf \{System Instruction\}} ========================
\end{center}

As an advanced reading comprehension and information retrieval agent, your task is to answer summative user questions based on retrieved documents. 
For each document, an [ID] is prepended to indicate the relative order in the document corpus. 
These documents provide a condensed version of the original long corpus to help you with summative tasks. 

Requirements: 

(1) Organize the documents to form a logically clear summary of the main content or plot of the original long corpus. 

(2) Your answer should be comprehensive and coherent. Your summative answer should be an independent paragraph, enabling the user to comprehend the content or plot of the original corpus. Do NOT include additional uncertain or made-up information in your answer.

(3) Your answer should be NO MORE THAN $\texttt{\{response\_max\_length\}}$ WORDS.

\end{tcolorbox}
\end{small}

\section{Supplemental Experiments} \label{experimentapp}

\begin{table*}[tp]
\caption{Retrieved documents and agent response of a multi-hop question example. ``\!{\small \textcolor{Green}{\faCheck}}\!'' and ``\!\!{\small \textcolor{OrangeRed}{\faTimes}}\!\!'' indicate supporting chunks and irrelevant chunks, respectively. Relevant tokens matched by BM25 are \hl{highlighted}.}
\vskip -0.1in
\label{tab:casestudy}
\begin{center}
\setlength{\tabcolsep}{3pt}
{\fontsize{6.8pt}{8pt}\selectfont
\begin{tabularx}{\linewidth}{>{\hsize=0.14\hsize}X>{\hsize=1.43\hsize}X>{\hsize=1.43\hsize}X}
\toprule[1.2pt]
Query & \multicolumn{2}{c}{{\it Who is the wife of the man who produced the documentary of the pop star who influenced Beyonc\'e?}} \\
\midrule \rowcolor{SkyBlue!20}
& \multicolumn{1}{c}{$\Psi$-RAG \!\!\faKey\!\! (1st attempt: original query)} & \multicolumn{1}{c}{$\Psi$-RAG \!\!\faKey\!\! (2nd attempt: {\it Who is the spouse of David Gest?})} \\
\midrule
\multirow{6}{*}{BM25} & \Cell{
\!\!\textcolor{OrangeRed}{\faTimes}\!\! 1: ``The Mistress'' features... a married \hl{man}, whose \hl{wife} was played by... \\
\!\!\textcolor{OrangeRed}{\faTimes}\!\! 2: Pizza Man is a 1991 comedy film \hl{starring} Bill Maher... \\
\!\!\textcolor{OrangeRed}{\faTimes}\!\! 3: Real Life is a 1979 American comedy film \hl{starring} Albert Brooks... \\ %portrays a \hl{documentary} filmmaker...
\!\!\textcolor{OrangeRed}{\faTimes}\!\! 4: No Escape is a 1936 British thriller film directed by... \\
\!\!\textcolor{Green}{\faCheck}\!\! 5: Michael Jackson: The Life of an Icon is a \hl{documentary} film about \hl{pop} singer... produced by his friend, {\bf David Gest} ...}
& \Cell{
\!\!\textcolor{Green}{\faCheck}\!\! 1: Liza and \hl{David}... featuring {\bf Liza Minnelli} and her then-husband, \hl{David Gest}. \\ 
\!\!\textcolor{Green}{\faCheck}\!\! 2: Michael Jackson: The Life of an Icon... produced by his friend, \hl{David Gest}... \\
\!\!\textcolor{OrangeRed}{\faTimes}\!\! 3: Dave Fennoy Fennoy in Phoenix, Arizona \hl{David} Henderson Fennoy... \\
\!\!\textcolor{OrangeRed}{\faTimes}\!\! 4: Chileab... was the second son of \hl{David}, King of Israel... \\
\!\!\textcolor{OrangeRed}{\faTimes}\!\! 5: Random Hearts is a 1999 American romantic drama film... their \hl{spouses} were having an affair... }\\
\midrule
\multirow{3}{*}{\Cell{1st\\layer}} & \Cell{
1: Beyonc\'e's music career, NBA players, American football, ... \\
2: film characters, biographies of actors, musical fantasy comedies, ... \\
3: DC Comics characters, superhero films, Marvel and DC comics, ... \\
% 4: HBO drama series, reality TV shows, American television series,... \\
% 5: English monarchs, medieval history, royal families,... 
}
& \Cell{
1: DC Comics characters, superhero films, Marvel and DC comics, ... \\
2: cricket tournaments, FIFA Confederations Cup, ICC Champions Trophy, ... \\
3: film characters, biographies of actors, musical fantasy comedies, ... \\
% 4: English monarchs, medieval history, royal families,... \\
% 5: Laos villages, Hong Kong areas, Thai subdistricts, ...\\
} \\
\midrule
\multirow{3}{*}{\Cell{Last\\abstract\\layer}} & \Cell{
1: Beyonc\'e's music career, R\&B and pop genres, female empowerment, ... \\
2: show business families, celebrity siblings, marriages and spouses, ... \\
3: Beyonc\'e's solo career, music awards, Billboard charts, ... \\
% 4: musical influences, Michael Jackson, Diana Ross,... \\
% 5: Beyonc\'e biography, Destiny's Child history, music career,... 
}
& \Cell{
1: Swedish personalities, biographical accounts, historical figures, ... \\
2: Game of Thrones, Justice League, Aquaman, ... \\
3: European royalty, medieval history, noble families, ... \\
% 4: voice actors, American nationality, video game series,... \\
% 5: Roman emperors, Julio-Claudian dynasty, family relationships,... \\
} \\
\midrule
\multirow{5}{*}{\Cell{Leaf\\layer}} & \Cell{
\!\!\textcolor{OrangeRed}{\faTimes}\!\! 1: Beyonc\'e's work has influenced numerous artists... \\
\!\!\textcolor{OrangeRed}{\faTimes}\!\! 2: On March 30, 2015, it was announced that Beyonc\'e is a co-owner... \\
\!\!\textcolor{Green}{\faCheck}\!\! 3: Beyonc\'e names {\bf Michael Jackson} as her major musical influence... \\
\!\!\textcolor{Green}{\faCheck}\!\! 4: Michael Jackson: The Life of an Icon is a documentary film... \\
\!\!\textcolor{OrangeRed}{\faTimes}\!\! 5: Beyonc\'e Giselle Knowles-Carter (/\textipa{bi:""'jOnseI}/ bee-YON-say)... }
& \Cell{
\!\!\textcolor{OrangeRed}{\faTimes}\!\! 1: Dave Fennoy Fennoy in Phoenix, Arizona... \\
\!\!\textcolor{OrangeRed}{\faTimes}\!\! 2: Gal Gadot - Varsano (Hebrew: ... \\
\!\!\textcolor{OrangeRed}{\faTimes}\!\! 3: Gal Gadot - Varsano (Hebrew: ... \\
\!\!\textcolor{OrangeRed}{\faTimes}\!\! 4: Gal Gadot - Varsano (Hebrew: ... \\
\!\!\textcolor{OrangeRed}{\faTimes}\!\! 5: Catharina Wallenstedt was the daughter of Bishop Laurentius... \\
} \\
\midrule
\multirow{5.5}{*}{\Cell{Final\\retrieved chunks}} & \Cell{
\!\!\textcolor{Green}{\faCheck}\!\! 1: Michael Jackson: The Life of an Icon... [score: $0.406$] \\
\!\!\textcolor{OrangeRed}{\faTimes}\!\! 2: Nine Days that Changed the World is a 2010 documentary... [score: $0.101$] \\
\!\!\textcolor{Green}{\faCheck}\!\! 3: Beyonc\'e names {\bf Michael Jackson} as her major musical influence... [score: $0.071$] \\
\!\!\textcolor{OrangeRed}{\faTimes}\!\! 4: Beyonc\'e's work has influenced numerous artists... [score: $0.064$] \\
\!\!\textcolor{OrangeRed}{\faTimes}\!\! 5: Following the disbandment of Destiny\'s Child... [score: $0.038$] }
& \Cell{
\!\!\textcolor{Green}{\faCheck}\!\! 1: Liza and David... featuring {\bf Liza Minnelli} and her then-husband, David Gest. [score: $0.980$] \\
\!\!\textcolor{Green}{\faCheck}\!\! 2: Michael Jackson: The Life of an Icon... [score: $0.201$] \\
\!\!\textcolor{OrangeRed}{\faTimes}\!\! 3: Gal Gadot - Varsano (Hebrew: ... [score: $0.029$] \\
\!\!\textcolor{OrangeRed}{\faTimes}\!\! 4: Gal Gadot - Varsano (Hebrew: ... [score: $0.026$] \\
\!\!\textcolor{OrangeRed}{\faTimes}\!\! 5: Gal Gadot - Varsano (Hebrew: ... [score: $0.020$]
} \\
\midrule
\multirow{4}{*}{\Cell{Agent\\response}} & \Cell{
According to ``Michael Jackson: The Life of an Icon'', David Gest is the friend and producer of the documentary film about pop singer Michael Jackson... I need more information on who David Gest's wife was at that time or if he had a wife. \\
$\texttt{<retrieve>}$ {\it Who is the spouse of David Gest?} }
& \Cell{
According to ``Michael Jackson: The Life of an Icon''... Lastly, I found that Liza Minnelli was the wife of David Gest at some point in time. \\
$\texttt{<answer>}$ {\it Liza Minnelli} \!\textcolor{Green}{\faCheck}\!\! }
\\
\bottomrule[1.2pt]
\end{tabularx}
}
\end{center}
\vskip -0.2in
\end{table*}

\begin{table*}[tp]
\caption{An example of minor class retrieval on MultiHop-RAG (\textcolor{BrickRed}{S[:50]}+\textcolor{Periwinkle}{E[:5]}). The top-1 summative abstract in each layer is listed. \textcolor{BrickRed}{Red} and \textcolor{Periwinkle}{purple} text relate to the major class ``\textcolor{BrickRed}{Sports}'' and the minor class ``\textcolor{Periwinkle}{Entertainment}'', respectively.}
\vskip -0.1in
\label{tab:casestudytree}
\begin{center}
\setlength{\tabcolsep}{3pt}
{\fontsize{6.8pt}{8pt}\selectfont
\begin{tabularx}{\linewidth}{>{\hsize=0.14\hsize}X>{\hsize=1.43\hsize}X>{\hsize=1.43\hsize}X}
\toprule[1.2pt]
Query & \multicolumn{2}{c}{\textcolor{Periwinkle}{\it What was the name of the tunnel where the fatal car crash occurred in The Crown?}} \\
\midrule \rowcolor{SkyBlue!20}
& \multicolumn{1}{c}{RAPTOR \cite{raptor}} & \multicolumn{1}{c}{$\Psi$-RAG \!\!\faFile*[regular]\!\!} \\
\midrule
\multirow{4}{*}{\Cell{1st\\layer}} & \Cell{
\textcolor{BrickRed}{Various sports teams face outcomes: Inter Miami eliminated, Italy and Australia in Davis Cup final, Rangers lead Diamondbacks 3-1, Lions beat Packers, 49ers defeat Cowboys, Vikings lose to Broncos and Bears,} \textcolor{Periwinkle}{amidst other results and controversies including Mattel's inaccurate Cherokee Barbie.}
}
& \Cell{
\textcolor{Periwinkle}{The Crown's final season covers Prince Charles and Camilla, Princess Diana, and Dodi Fayed. Diana's death and her sons' reactions are depicted. Separate stories include Jada Pinkett Smith and Will Smith's marriage, Mattel's Barbie honoring Wilma Mankiller, and Dodi Fayed's past engagement to Kelly Fisher.}
} \\
\midrule
\multirow{5}{*}{\Cell{2nd\\layer}} & \Cell{
\textcolor{BrickRed}{Inter Miami loses 1-0 to FC Cincinnati, eliminating them from MLS playoff contention, despite Lionel Messi's 35-minute appearance off the bench, as Alvaro Barreal scores late, while other teams like Manchester United struggle in Europe with defensive issues and Chelsea's Reece James faces an abusive conduct charge.}
}
& \Cell{
\textcolor{Periwinkle}{Season six focuses on Prince Charles and Camilla, Princess Diana and Dodi Fayed. After Diana's 1997 death, William went missing at Balmoral, returning 14 hours later. Before her death, Diana told William and Harry not to believe papers and planned life changes, but a planned meeting never occurred due to her fatal car crash in Paris.}
} \\
\midrule
\multirow{6}{*}{\Cell{Leaf\\layer}} & \Cell{
\textcolor{BrickRed}{So, our mistakes can never happen in the Champions League, you get punished for it, so we have to learn from that '' However, Aston Villa legend Gabby Agbonlahor insists Onana needs to take a significant chunk of the blame for United's collapse on Wednesday night ``I feel sorry for Onana,'' Agbonlahor told talkSPORT  ``There's a clip of him throwing his gloves to the ground after the game  He's angry with himself}
}
& \Cell{
\textcolor{Periwinkle}{In both the series and in real-life, Diana and Dodi - along with their chauffeur, Henri Paul - were killed on 31 August 1997 when their Mercedes-Benz was pursued by paparazzi and crashed in the {\bf Pont de l'Alma tunnel} in Paris  The mother of two was 36 years old The fatal car crash was recreated in the very first scene of season six}
} \\
\midrule
\multirow{7}{*}{\Cell{Agent\\response}} & \Cell{
The retrieved documents once again do not mention anything about a tunnel or a fatal car crash from ``The Crown''. They are still discussing football, specifically Aston Villa and Manchester United. Since I have run out of retrieval attempts and the provided information does not answer the question. \\
$\texttt{<answer>}$ {\it Not mentioned} 
\!\!\textcolor{OrangeRed}{\faTimes}\!\! }
& \Cell{
To answer this user question, I need to identify the location and details regarding a specific event depicted in ``The Crown''. According to the retrieved document, Diana, Dodi, their chauffeur Henri Paul were involved in a fatal car crash. This incident occurred when they were pursued by paparazzi while driving through Paris. Their Mercedes-Benz crashed inside the Pont de l'Alma tunnel on 31 August 1997. \\
$\texttt{<answer>}$ {\it Pont de l'Alma} \!\textcolor{Green}{\faCheck}\!\! }
\\
\bottomrule[1.2pt]
\end{tabularx}
}
\end{center}
\end{table*}

\subsection{Case Study} \label{casestudy}

We present a multi-hop example from MuSiQue in \cref{tab:casestudy} to intuitively illustrate how $\Psi$-RAG performs multi-granular agentic retrieval on the abstract tree. The question, ``{\it Who is the wife of the man who produced the documentary of the pop star who influenced Beyonc\'e?}'' requires up to 3 hops to answer. %``Beyonc\'e $\rightarrow$ Michael Jackson $\rightarrow$ David Gest $\rightarrow$ Liza Minnelli''. 
The example shows that the keywords retrieved by $\Psi$-RAG \!\!{\small \faKey}\! at the first layer are relatively coarse, covering broad themes such as music, sports, and movies. 
% However, some terminological keywords (e.g., ``Beyonc\'e's music career'') help $\Psi$-RAG locate the relevant abstract at the beginning of the retrieval. 
With the search going deeper into the tree, the keywords become finer-grained. Consequently, two semantically related retrieved chunks tell that the required ``man who produced the documentary'' is David Gest. 

However, due to the structural isolation, we still do not know who is David Gest's wife, so the R\&A agent calls for a second retrieval attempt. 
As the second question lacks sufficient context, the coarse abstraction in the upper tree layers confuses the retriever, leading it to irrelevant chunks. In contrast, sparse retrieval can provide straightforward facts about David Gest by simply searching for documents with the keyword. As a result, the retrieved chunk leads to the final answer ``Liza Minnelli''. 
This demonstrates the importance of synergy between $\Psi$-RAG's hierarchical tree index and the multi-granular hybrid retriever, the key to its effectiveness. 

\subsection{Analyses on Hierarchical Abstract Tree} \label{treeablation}

{\bf Case study: retrieval on skewed data.} 
To further demonstrate the impact of the uniform effect and distribution adaptability, we conduct a case study based on the visualization study in \cref{skewvisual}. 
We compare the agent response processing a minor-class question on both tree indexes of RAPTOR and $\Psi$-RAG. From \cref{tab:casestudytree}, the 1st layer abstract nodes of RAPTOR are confused because of the uniform effect, with most of the contents related to the major class ``\textcolor{BrickRed}{Sports}''. This significantly interrupts the retriever which selects a major class abstract node in the 2nd layer. On the contrary, $\Psi$-RAG's minor abstract node at the 1st layer clearly points to the relevant topic ``The Crown''. This assists the retriever in finding the correct chunk. 

\begin{minipage}[t]{0.54\textwidth}
\vspace{0pt}
\centering
\captionof{table}{Comparative results of the indexes. ``SR'', ``R\&A'', and ``Q3R'' refer to the sparse retriever, the R\&A agent, and the reranker, specifically.}
\begin{small}
\adjustbox{max width=\linewidth}{
\begin{tabular}{lll}
\toprule[1.2pt]
Index & Single retrieval time (s) & F1 (\%) \\
\midrule \rowcolor{SkyBlue!20}
\multicolumn{3}{c}{{\it MuSiQue}} \\
Flat & $33.39$ & $30.50$ \\
GMM & $0.16$ & $17.17$ \\
HAT & $\mathbf{0.15}$ & $18.62$ (\textcolor{Green}{$\uparrow 1.45$}) \\
GMM + SR + R\&A + Q3R & $1.17$ & $47.26$ \\
HAT + SR + R\&A + Q3R & $1.12$ & $\mathbf{48.90}$ (\textcolor{Green}{$\uparrow 1.64$}) \\
\midrule \rowcolor{SkyBlue!20}
\multicolumn{3}{c}{{\it MultiHop-RAG}} \\
Flat & $62.21$ & $42.44$ \\
GMM & $0.17$ & $35.81$ \\
HAT & $\mathbf{0.16}$ & $36.30$ (\textcolor{Green}{$\uparrow 0.49$}) \\
GMM + SR + R\&A + Q3R & $0.66$ & $51.32$ \\
HAT + SR + R\&A + Q3R & $0.62$ & $\mathbf{56.23}$ (\textcolor{Green}{$\uparrow 4.91$}) \\
\bottomrule[1.2pt]
\end{tabular}
}
\label{tab:index}
\end{small}
\end{minipage}
\hfill
\begin{minipage}[t]{0.4\textwidth}
\vspace{0pt}
\centering
\includegraphics[width=\textwidth]{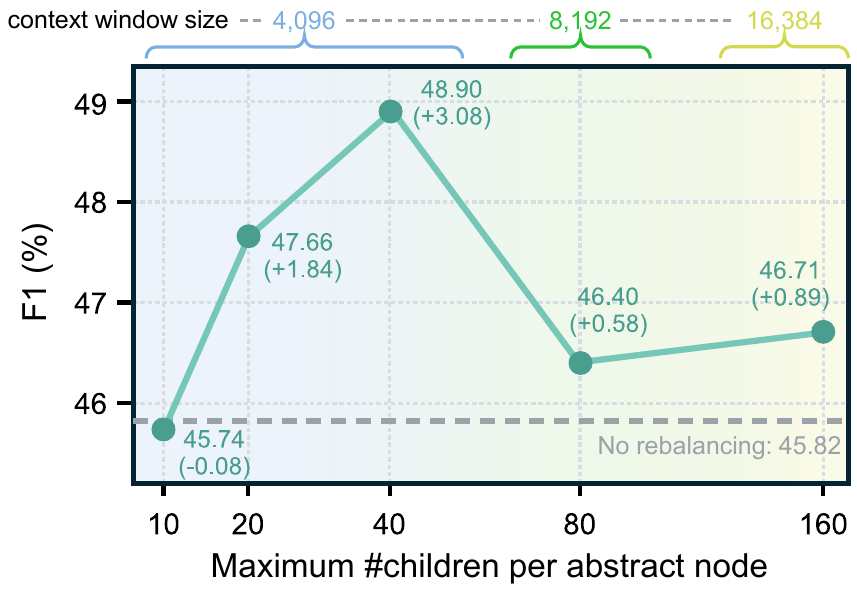}
\captionof{figure}{Sensitivity of $\Psi$-RAG to the maximum number of children per abstract node on MuSiQue.}
\label{fig:rebalance}
\vspace{0.1in}
\end{minipage}

{\bf Quantitative study on tree indexes.} 
We conduct a quantitative comparison between the hierarchical abstract tree (HAT) of $\Psi$-RAG \!\!{\small \faFile*[regular]}\!\! and (1) a flat index, which performs brute-force sequential search on all document chunks with $O(n)$ time; (2) the GMM-based cluster tree of RAPTOR~\cite{raptor}, both without and with the proposed multi-granular agentic retriever modules. 
From \cref{tab:index}, we make the following observations:

(1) {\bf Flat dense retrieval is much slower than tree retrieval.} Although dense retrieval on a flat index serves as a performance upper bound for tree indexes, its linear time complexity incurs at least ${\sim}100\times$ higher runtime than the logarithmic retrieval of tree indexes. 

(2) {\bf HAT outperforms the GMM tree both with or without the agentic retriever modules.} The performance increase points to the minor class distribution in the data, which characterizes the contribution of our tree index to alleviating poor distribution adaptability.

{\bf Effect of rebalancing.} We have set the maximum number of children per abstract node for the abstract tree. This keeps the total abstract length within the context window of the abstraction agent. 
\cref{fig:rebalance} shows the sensitivity of $\Psi$-RAG to the rebalancing process after tree construction, which lead to three observations. (1) Rebalancing without excessive splitting helps build a better tree structure with higher performance. (2) Splitting an abstract node too many times may lead to relevant node missing during top-$k$ dense vector matching: setting maximum \#children to 10 makes the F1 lower than that without rebalancing. (3) When the abstracts become too long (8k--16k tokens), their semantics become more vague which also reduces the accuracy.

\subsection{Analyses on Multi-granular Agentic Retriever} \label{retrieveablation}

\begin{figure*}[t]
  \begin{center}
    \centerline{\includegraphics[width=\textwidth]{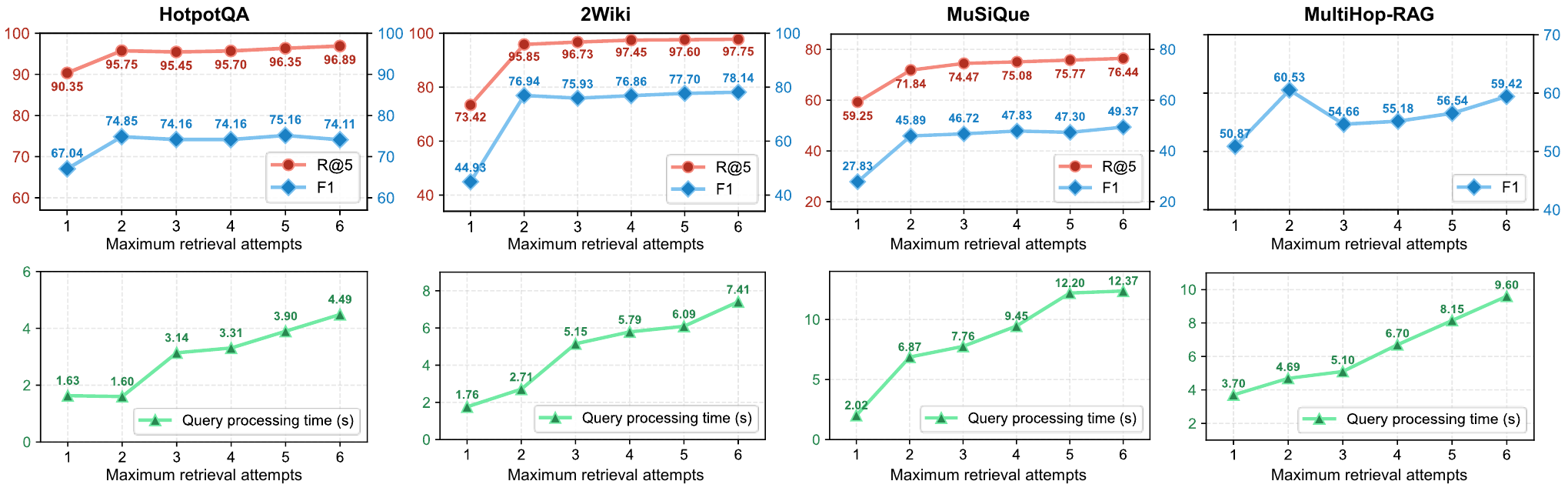}}
    \caption{Retrieval and QA performance in Recall@5 (R@5\%) and F1 (\%), and average processing time for each query (s) under concurrent execution, {\it w.r.t.} the maximum number of retrieval attempts.
    }
    \label{fig:attemptcurve}
  \end{center}
  \vskip -0.2in
\end{figure*}

{\bf Maximum number of retrieval attempts.} 
We evaluate how the efficiency and performance of the R\&A agent vary with the maximum number of iterative retrieval attempts ($i_{\max}+1$), as shown in \cref{fig:attemptcurve}. The results indicate that on multi-hop datasets, the first extra retriever call yields the greatest performance gain, while subsequent attempts contribute minimal or even negative improvement. Each retrieval attempt introduces a relatively constant time overhead, and longer reasoning chains tend to introduce unnecessary noise \cite{longreasonchain}. Therefore, it is crucial to strike a balance between the cost of retrieval iterations and the performance gains across different application scenarios.

\begin{figure*}[t]
  \begin{center}
    \centerline{\includegraphics[width=\textwidth]{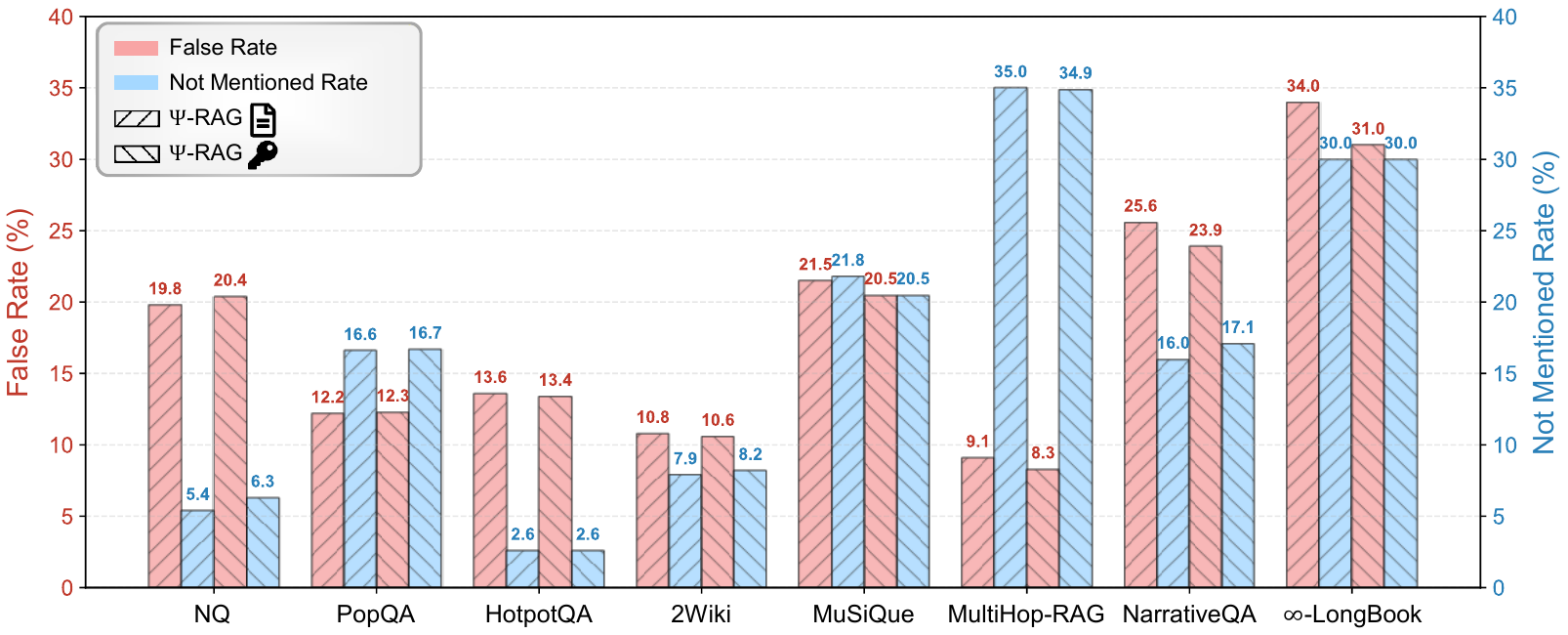}}
    \caption{
False Rate and Not Mentioned Rate of $\Psi$-RAG on different datasets. 
    }
    \label{fig:failure}
  \end{center}
  \vskip -0.2in
\end{figure*}

\subsection{Failure Analysis} \label{failure}

We analyze the failure cases of $\Psi$-RAG by measuring the False Rate (rate of incorrect answers) and Not Mentioned Rate (rate of retrieval failures). An incorrect answer is defined as sharing no overlapping tokens with the ground truth. \cref{fig:failure} reveals the following key patterns:

(1) {\bf Retrieval failures stem from multiple factors.} Index scale is a primary factor: the Not Mentioned Rate for $\infty$-LongBook is significantly higher than that for NarrativeQA, which has an index only 30\% the size on average. This is because upper-level abstracts in a larger search space are coarser and introduce more noise. Another factor is query complexity: datasets with higher maximum hop counts (MuSiQue with 4 hops) exhibit higher Not Mentioned Rates than those with fewer hops (HotpotQA and 2Wiki with 2 hops). Although we increase the maximum retrieval attempts for complex questions, we observe that if the initial retrieval for a reorganized query fails, subsequent retrieval attempts are also highly likely to fail. This is because the reorganized query often does not shift the high-level semantic focus, causing the tree retriever to revisit the same abstract nodes. Meanwhile, the high-frequency entities in the reorganized query lead the sparse retriever to return many noisy chunks. As shown in \cref{casestudy}, BM25 retrieves many irrelevant chunks with ``David'' that are unrelated to the target ``David Gest''. 

(2) {\bf The False Rate consistently exceeds the Not Mentioned Rate.} This occurs on most datasets excluding MultiHop-RAG (with unanswerable negative samples). A manual analysis of 50 failure cases from L3-70B on MuSiQue identifies the following causes: (\romannumeral 1) LLM misguidance or guesses resulting from incomplete retrieval (29 cases); (\romannumeral 2) logical errors in agentic reasoning (8 cases); (\romannumeral 3) LLM hallucination (6 cases); (\romannumeral 4) granularity mismatch between the answer and the ground truth (7 cases). Therefore, the primary cause of a higher false rate is missing key document chunks in retrieval.

\section{Real-World Extensions} \label{extensions}

\subsection{Approximate Nearest Neighbor Search for Efficient Indexing}

We extend $\Psi$-RAG with two efficient techniques for very large corpora with 10M+ and even 100M+ tokens. They include:

\begin{itemize}[leftmargin=10pt,itemsep=1pt,topsep=2pt]
\item {\bf Bucketing}: following the inverted-file index in Faiss~\cite{faiss}, we first partition a large corpus into multiple buckets using a fast recursive spherical $k$-means. We then build a tree inside each bucket and finally merge them at the highest layer. Assuming $b \ll n$ buckets of balanced size, bucketing reduces the average indexing time complexity from $O(n^2)$ to $O(n^2/b)$ and space complexity from $O(n^2)$ to $O(n^2/b^2)$.
\item {\bf ANN search with HNSW}~\cite{hnsw}: we build an HNSW graph over all chunk embeddings and use the top-$k$ graph neighbor pairs of each chunk, which are then ranked to build the tree. This reduces the average time complexity of similarity ranking from $O(n^2 \log n)$ to $O(kn \log kn)$ and space complexity from $O(n^2)$ to $O(kn)$.
\end{itemize}

We evaluate the efficiency of the two extensions on two large subsets from Wiki18 (${\sim}$3B tokens) in FlashRAG~\cite{flashrag}: Wiki18-6M (6.1M tokens, ${\sim}$42k chunks) and Wiki18-50M (50.8M tokens, ${\sim}$350k chunks). The time and memory costs are listed in \cref{tab:annefficiency}. It shows that HNSW can generate sorted document pairs for a corpus of 50M tokens with less than 1.5 GB of storage. Combined with bucketing, tree indexing can be completed with MB-level memory usage and at less than 1\% of the original time cost. Therefore, LLM-based abstraction becomes the primary efficiency bottleneck for scaling the framework to industrial corpora. Exploring efficient alternatives to LLM-based abstraction will be a key direction for our future work.

\begin{table}[h]
\centering
\caption{Time and memory comparison of $\Psi$-RAG's extensions on large-scale Wiki18 subsets. Embedding dimension is set to 1,024. Experiments are conducted on 8 CPU cores (Wiki18-6M) / 64 CPU cores (Wiki18-50M) and 256GB Memory.}
\label{tab:annefficiency}
\begin{small}
\adjustbox{max width=\linewidth}{
\begin{tabular}{llll|ll}
\toprule[1.2pt]
Data & \multicolumn{3}{c}{Wiki18-6M} & \multicolumn{2}{c}{Wiki18-50M} \\
\cmidrule(lr){2-4} \cmidrule(lr){5-6}
Method 
& Vanilla $\Psi$-RAG 
& + HNSW 
& + Bucketing ($b=4$) + HNSW
& + HNSW 
& + Bucketing ($b=63$) + HNSW \\
\midrule
Embedding time (s) 
& \multicolumn{3}{l|}{$1{,}229$}
& \multicolumn{2}{l}{$10{,}208$}
\\

Similarity ranking time (s) 
& $304.84$ 
& $55.99$ (\textcolor{Green}{$5.44\times$}) 
& $\mathbf{0.89}$ (\textcolor{Green}{$342.5\times$}) 
& $168.02$ 
& $\mathbf{7.87}$ (\textcolor{Green}{$21.35\times$}) \\

Similarity matrix memory 
& $13.16\mathrm{GB}$ 
& $10.55\mathrm{MB}$ (\textcolor{Green}{$1{,}277\times$}) 
& $\mathbf{2.45\mathrm{MB}}$ (\textcolor{Green}{$5{,}483\times$}) 
& $88.17\mathrm{MB}$ 
& $\mathbf{2.15\mathrm{MB}}$ (\textcolor{Green}{$41.01\times$}) \\

Similarity ranking memory 
& $65.81\mathrm{GB}$ 
& $656.7\mathrm{MB}$ (\textcolor{Green}{$102.6\times$}) 
& $\mathbf{37.89\mathrm{MB}}$ (\textcolor{Green}{$1{,}778\times$}) 
& $1.34\mathrm{GB}$ 
& $\mathbf{34.10\mathrm{MB}}$ (\textcolor{Green}{$40.24\times$}) \\

Tree indexing time (s) 
& $15{,}996$ 
& $13{,}957$ (\textcolor{Green}{$1.15\times$}) 
& $\mathbf{848}$ (\textcolor{Green}{$18.86\times$}) 
& ${\sim} 280{,}000$ 
& $\mathbf{2{,}261}$ (\textcolor{Green}{${\sim}124\times$}) \\

Abstraction time (s) 
& ${\sim} 205{,}000$ 
& ${\sim} 200{,}000$ 
& ${\sim} 89{,}400$ 
& ${\sim} 1{,}800{,}000$ 
& ${\sim} 1{,}400{,}000$ \\
\bottomrule[1.2pt]
\end{tabular}
}
\end{small}
\end{table}

\subsection{Query Hop Discriminator}

Our evaluation utilizes the maximum hop count of each dataset for the maximum number of iterative retrieval time $i_{\max}$. 
\begin{wrapfigure}[7]{r}{0.3\textwidth}
\centering
\captionof{table}{Multi-hop QA results of $\Psi$-RAG with the query hop discriminator.}
\label{tab:qhd}
\adjustbox{max width=\linewidth}{
\begin{tabular}{lcc}
\toprule[1.2pt]
Dataset & R@5 (\%) & F1 (\%) \\
\midrule
HotpotQA & $95.05$ & $74.27$ \\
MuSiQue & $72.68$ & $46.63$ \\
\bottomrule[1.2pt]
\end{tabular}
}
\end{wrapfigure}
This limits the applicability in real-world datasets with diverse user queries. 
For practical settings without prior maximum hop count, we implement a lightweight query hop discriminator $\mathcal{Q}: q \mapsto i_{\max}$ using a 2-layer MLP trained on questions with 1 to 4 hops. It predicts the hops of each user query before retrieval to automatically decide the maximum retrieval time for the R\&A agent with very small latency. \cref{tab:qhd} shows the QA results with $\mathcal{Q}$ on some multi-hop datasets\footnote{We train $\mathcal{Q}$ with user questions from datasets other than the indexed corpus: NQ, 2Wiki, MultiHop-RAG without negative samples, and MuSiQue for indexing HotpotQA (HotpotQA for indexing MuSiQue). The lengths of decomposed sub-query lists or the number of evidence passages serve as labels.}. Results indicate that the query hop discriminator preserves strong effectiveness of $\Psi$-RAG without the prior maximum hop count.

\section{Related Work} \label{relatedwork}

{\bf Retrieval-Augmented Generation} \cite{rag, ragsurvey}. Based on the type of index and retriever, RAG frameworks can be categorized as: (1) {\it Sparse retrieval}: methods such as TF-IDF \cite{tfidf} and BM25 \cite{bm25} retrieve documents where keywords from the user question appear more often in the corpus. They are fast and suitable for token-level detailed search but less capable of capturing deep semantic meanings. 
(2) {\it Dense vector retrieval} \cite{dpr}: this approach utilizes powerful language embedding models \cite{colbert, contriever, nvembed, qwen3embedding} 
to map queries and documents into a continuous vector space \cite{ragsurvey}. 
It retrieves the top-$k$ documents most relevant to the user question using vector metrics like cosine similarity. Sometimes, embedding-based rerankers \cite{bge, qwen3embedding} are also utilized to rerank retrieved documents based on user preference, further improving retrieval quality. Dense vector-based RAG excels at capturing deep semantics and key features of entities, making it suitable for more complex QA scenarios. Nonetheless, dense retrieval could prefer documents highly relevant in features but totally unrelated to target entities.
(3) {\it Hybrid retrieval}: some approaches integrate both sparse and dense retrieval to combine their strengths, leading to a number of powerful RAG frameworks \cite{clear, roc, hhr, hirag}.

{\bf Agentic RAG.} 
The rapid advancement in the intelligence of LLMs has heightened focus on their reasoning capabilities and demands on RAG frameworks. 
Recent efforts \cite{react, ircot, selfrag, searcho1} incorporate specifically designed LLM agents to integrate retrieval into LLM reasoning process. Specifically, the agent recursively decomposes a question into multiple sub-questions and invokes the retriever at each step \cite{agenticragsurvey}. This enables the LLM to autonomously and flexibly handle multi-hop questions. However, most of them are built upon flattened linear or inverted indexes for token-level information. 

{\bf Graph-RAG.} 
As the most popular form of structured RAG, knowledge graph indexes \cite{graphrag, graphragsurvey, graphragbench} can capture semantic relationships between document chunks or entities. They often build a unified index for an entire corpus using named entity recognition and extraction methods like the Open Information Extraction (OpenIE) \cite{openie}. Cross-document multi-hop associations are then captured by walking from node to node \cite{hipporag, hipporag2}, extracting shortest paths \cite{higraphrag, pathrag}, or extracting subgraphs \cite{lightrag, subgraphrag}. However, graph indexes are less effective at capturing hierarchy. 
Some Graph-RAG methods have to resort to less efficient compensatory strategies to handle summative tasks, including but not limited to hierarchical graphs \cite{hnsw, higraphrag}, community summaries \cite{graphrag, archrag}, LLM-based hierarchical filtering \cite{lightrag}, or graph neural network-based message aggregation \cite{gor}. Coupled with the significant overhead of OpenIE, efficient index construction remains a key challenge for Graph-RAG.

{\bf Tree-RAG.}
Tree indexes are introduced to explicitly model the hierarchical relationships of documents while achieving fast retrieval with logarithmic complexity \cite{llamaindex}. Traditional tree indexes for nearest neighbor search include top-down splitting \cite{kdtree, balltree} and bottom-up hierarchical clustering \cite{singlelinkage, upgma, hc-retrieval}. 
However, traditional algorithms offer limited generalization. Most existing dense vector-based Tree-RAG methods \cite{dhr, llamaindex, longrefiner, e2graphrag, trag, treerag} rely on explicit document structure (session titles, tables of contents, etc.) or contextual relationships between adjacent text chunks. As a result, they are not directly applicable to documents lacking explicit structure or to cross-document retrieval.
RAPTOR \cite{raptor} is the first passage-level Tree-RAG framework designed for unstructured long documents. It recursively clusters document chunks via Gaussian Mixture Model (GMM) and generates a summary for each cluster. MemTree \cite{memtree} further adapts Tree-RAG to the long-term LLM memory scenario, but both of them are limited to the passage-level indexes and the single-document retrieval setting. 
To our knowledge, the only corpus-level Tree-RAG is LATTICE \cite{lattice}, which performs beam search retrieval using an LLM-evaluated path relevance score. However, there are no discussions towards the limitations on multi-hop QA, and its tree building process is not open-sourced at the time of writing. 

\section{Discussion and Future Work}

{\bf Updating the index.} Real-time document insertion and index updates are common requirements for RAG applications, yet they pose challenges for traditional Tree-RAG methods, as adding new documents necessitates re-clustering the entire corpus. Our abstract tree index is constructed based on local pairwise similarities, making it feasible to be updated with new documents.

We propose an easy-to-implement approach for the insertion of $\Delta n$ new document chunks. A single new node is collapsed to the leaf layer by searching for its most similar leaf in the tree via dense vector matching. Then, for keyword abstracts, $\Psi$-RAG traverses upstream to find the parent abstract nodes of the new node (or maintains an $O(n)$ parent index table for faster lookups) and appends new keywords to each parent. For summative abstracts, the abstraction LLM is invoked to regenerate the abstract for each parent node, similar to existing abstract updating strategies \cite{memtree}. This costs up to $O(n\Delta n)$ time. 
For batch insertion of $\Delta n$ chunks, a tree can first be built with only new chunks and then integrated into the main branch via abstract node collapse. If it is collapsed into the $l$th layer with $n_l$ nodes, then this updating costs $O((\Delta n)^2 n_l \log \Delta n)$ which is much smaller than that of re-indexing. 

We note that repeatedly inserting isolated nodes under the same parent may amplify the semantic or granularity biases between them. Considering the semantic coherence in the real-world textual data, we thus recommend inserting new contents as a subtree rather than as individual nodes when they differ substantially from the existing tree. 

{\bf Future work.}
An efficiency challenge for $\Psi$-RAG and other passage-level Tree-RAG lies in the abstraction LLM calls during indexing. 
Developing a more efficient abstraction scheme is a key direction for scaling Tree-RAG to industrial applications. 
Moreover, as analyzed in \cref{failure}, sparse retrievers like BM25 resort to the traditional named entity recognition technique, neglecting the high co-occurrence frequency of named entities. 
Another future work involves adapting agents to more specialized and diverse user needs with various post-training techniques, such as reinforcement learning and rejection sampling-based supervised fine-tuning.

\end{document}